\documentclass[runningheads]{llncs}

\usepackage{eccv}

\usepackage{eccvabbrv}

\usepackage{graphicx}
\usepackage{booktabs}
\usepackage{xcolor}
\usepackage{wrapfig}
\usepackage[accsupp]{axessibility}  

\usepackage[pagebackref,breaklinks,colorlinks,citecolor=eccvblue]{hyperref}
\usepackage{booktabs}
\usepackage{multirow}
\usepackage{graphicx}
\usepackage{fix-cm} 
\usepackage{orcidlink}

\begin{document}

\title{Personalize Your Large Vision-language Models with In-context Prompt Tuning} 

\titlerunning{ICPT}

\author{Yanshu Li\inst{1} \and
Jiaqian Li\inst{1} \and
Kuai Yu\inst{2} \and Xi Xiao\inst{3} \and Dongfang Liu \inst{4} \and Tianyang Wang \inst{3} \and Ruixiang Tang \inst{5}\thanks{Corresponding author.}}

\authorrunning{Y Li, J Li et al.}

\institute{Brown University, USA\\
\email{yanshu\_li1@brown.edu} \and
Columbia University, USA\and
University of Alabama at Birmingham, USA \and
Purdue University, USA \and
Rutgers University, USA\\
\email{ruixiang.tang@rutgers.edu}}

\maketitle

\begin{abstract}
Large vision-language models (LVLMs) have demonstrated strong general multimodal capability and are increasingly deployed in downstream systems. This trend has driven growing interest in LVLM personalization, which aims to enable models to quickly and effectively learn out-of-distribution multimodal concepts to meet user-specific needs. However, many existing methods rely on inference-time training, which reduces efficiency. They also struggle to maintain accuracy in complex multi-image, multi-concept settings. These limitations restrict the broader deployment of LVLM-based systems. Therefore, this paper proposes in-context prompt tuning (ICPT). Specifically, ICPT employs a lightweight projection module capable of operating in complex scenarios to extract fine-grained visual semantics from multiple reference images, seamlessly transforming these features alongside identity-label mappings into continuous prompts. To maximize computational efficiency, this module adaptively determines the prompt length based on the intrinsic visual complexity of each concept. Crucially, to overcome the environmental biases and cross-concept interference prevalent in real-world applications, we introduce two novel geometric regularizations. These constraints refine prompt representations by decoupling key identities from transient environmental states and separating concepts to avoid semantic confusion. Extensive experiments show that ICPT achieves state-of-the-art personalization accuracy across diverse tasks and LVLM backbones. 
\end{abstract}
\section{Introduction}
A large vision-language model (LVLM) integrates a vision encoder with a large language model (LLM) backbone and is trained using staged learning procedures \cite{mm}. By extending the capability boundary of LLMs to more real-world scenarios, LVLMs have become a core component in many modern AI systems, including interactive assistants and specialized agents \cite{app3,assis,app4}. 

As these systems, which are built upon general-purpose LVLMs, naturally shift toward private and user-facing deployments, \textit{\textbf{personalization}} emerges as a critical research topic \cite{pers2}. Enabling models to acquire user-specific concepts that go beyond their parametric knowledge is a key step toward fully unlocking the practical value of LVLM-based applications. As illustrated in Fig. \ref{fig:teaser}, by learning the conceptual mapping between a personalized label $<$sks1$>$ and the visual features of a specific individual given by a few reference images, an LVLM can subsequently leverage this concept directly to perform various tasks.
\begin{figure}[tb]
  \centering
  \includegraphics[width=1\linewidth]{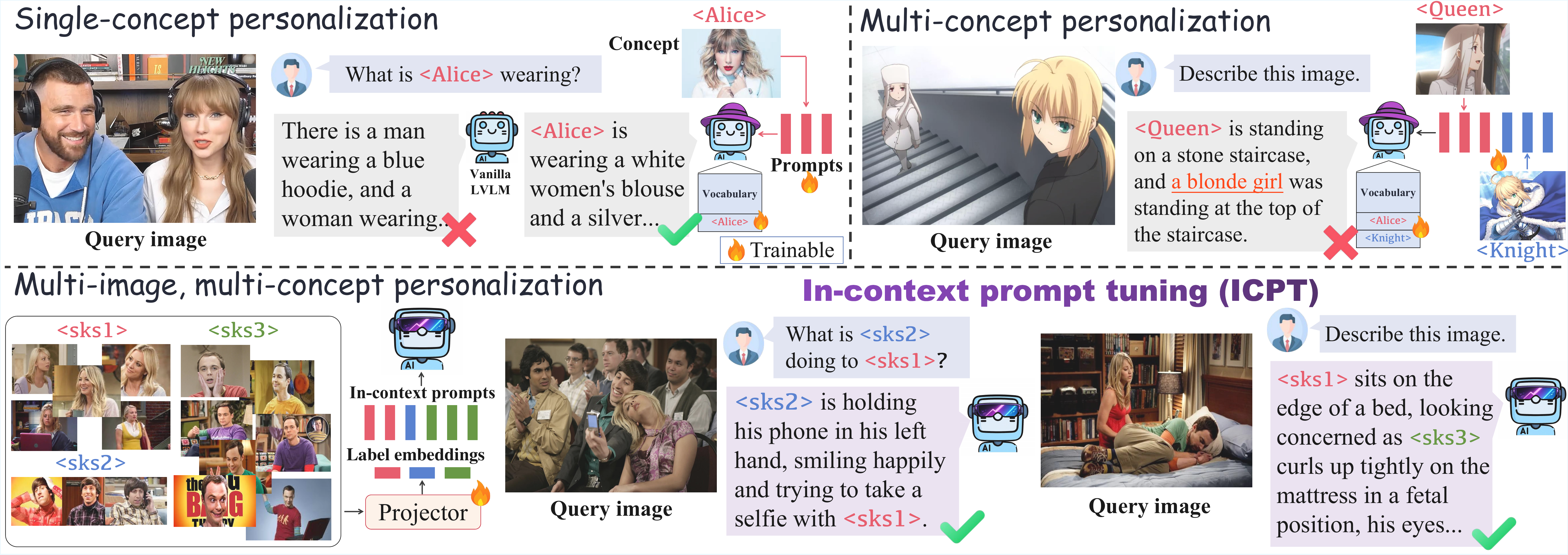}
  \caption{Overview of LVLM personalization. Existing methods often rely on vocabulary expansion with inference-time training and struggle in multi-image and multi-concept settings. Our proposed ICPT improves efficiency and performance in complex scenarios.}
  \label{fig:teaser}
\end{figure}
However, achieving stable and efficient personalization in LVLMs remains challenging, as it requires learning both intra- and cross-modal mappings. A common solution uses in-context learning (ICL) by injecting explicit concepts into the prompt \cite{icl2}. However, LVLMs are highly sensitive to few-shot prompts, which often leads to unstable concept acquisition \cite{icl1}. Moreover, concept images are processed directly as visual inputs, generating many redundant tokens that reduce efficiency \cite{icl3}. To address these challenges, two main lines of work have been explored. One of them focuses on post-training the entire model to improve its ICL capacity, such as PVIT \cite{pvit}. However, the cost of building large-scale training datasets and performing model training further exacerbates efficiency concerns.

In contrast, soft prompt-based personalization becomes a dominant paradigm. Methods such as MyVLM \cite{myvlm} and Yo'LLaVA \cite{yollava} learn a small set of prompts to encode user-specific concepts, enabling lightweight adaptation \cite{mcllava}. However, they reduce practical scalability, as each new concept requires vocabulary expansion and separate training. To avoid test-time training, PLVM \cite{plvm} utilizes an additional pretrained vision encoder to extract features from a reference image and trains a cross-attention module to combine them with the concept label for prompt generation. While offering novel insights, PLVM and recent work are evaluated only on relatively simple tasks and do not address two complex yet crucial scenarios: (1) \textbf{multi-concept} personalization, where the model must simultaneously learn and reason over multiple concepts spanning diverse categories beyond people, such as animals, objects, and scenes. (2) \textbf{multi-image} settings, where each concept is defined by multiple reference images exhibiting varying degrees of visual diversity, often requiring fine-grained and cross-image visual feature extraction. This motivates the following question: \textit{How can we generate prompts that encode multimodal concepts with sufficient clarity for reliable reasoning in these complex settings, while avoiding inference-time training?}

In pursuit of this question, we propose in-context prompt tuning (ICPT), a personalization method that simulates ICL entirely within the representation space of frozen LVLMs. At the core of ICPT is the Adaptive Concept Projector (ACP), which extracts hierarchical multi-scale features directly from the built-in vision encoder to generate a dedicated textual label embedding alongside a continuous visual prompt for each novel concept. To balance representational capacity with inference efficiency, the ACP incorporates a Dynamic Token Router (DTR) that adaptively allocates variable prompt lengths based on the inherent visual complexity of the target concept. To guarantee pure and robust identity acquisition through these prompts, we propose two structural constraints: Contextual Variation Memory (CVM) and Margin-constrained Concept Separation (MCS). CVM maintains a memory queue of environment-induced variations, allowing learned prompts to be projected into a bounded, state-invariant subspace. Meanwhile, MCS applies a dual-modality soft-margin constraint on visual tokens and text anchors to separate identities while preserving shared semantic priors. Optimized end-to-end with carefully curated data and training objectives designed for the proposed modules and constraints, ICPT seamlessly integrates into various LVLMs and enables efficient and robust personalization at inference time. We evaluate ICPT on four common personalization tasks under both single-concept and multi-concept settings. Results across four LVLMs with different sizes and backbones demonstrate the consistently state-of-the-art personalization capability of ICPT compared to ICL and prior methods. Therefore, ICPT is an effective and practical method for future LVLM personalization.

The contributions of this paper are three-fold. \textbf{First}, we propose the Adaptive Concept Projector to extract, aggregate, and transform fine-grained information from multiple reference images into continuous prompts natively usable by LVLMs, introducing a dynamic token routing mechanism to adaptively allocate prompt length. \textbf{Second}, to optimize the semantic representation of these prompts for complex scenarios, we propose two innovative structural constraints alongside a targeted training strategy to optimize the entire framework end-to-end. \textbf{Third}, extensive experiments demonstrate that ICPT achieves superior performance and broad generalization across diverse multimodal tasks and LVLM backbones compared to existing methods. 
\section{Related Works}
\textbf{Large Vision-language Models (LVLMs).} Most mainstream LVLMs consist of three core components: a vision encoder, a projector, and an LLM serving as the decoder \cite{mllm}. Given multimodal inputs containing both images and text, the vision encoder, such as the Vision Transformer (ViT) used in CLIP \cite{clip}, first partitions images into patches and converts them into visual tokens that encode perceptual features. These tokens are then merged and mapped by an MLP-based projector into the same latent space as textual tokens. They are then jointly processed by the LLM for unified multimodal reasoning \cite{mm}. With the progression from LLaVA1.5 \cite{llava} to more recent models like InternVL3.5 \cite{intern3.5}, and Qwen3VL \cite{Qwen3}, LVLMs have demonstrated steadily improving visual perception and understanding capabilities and are increasingly adopted as core components in a wide range of downstream applications \cite{app1,app2}. As a result, enabling effective personalization for LVLMs has become increasingly important \cite{pers}.

\textbf{Personalized LVLMs.} Personalization aims to enable LVLMs to recognize and utilize user-specific concepts, supporting rapid semantic alignment in private and personalized scenarios \cite{pers1}. In-context learning (ICL) is the most commonly used strategy, with methods such as LaMP \cite{lamp} and RAP \cite{rap} leveraging retrieval augmented generation for personalization. However, in multimodal settings, ICL often exhibits instability. PVIT \cite{pvit} addresses this issue by introducing additional post-training to improve the ICL capability of LVLMs. To achieve more lightweight personalization, MyVLM \cite{myvlm} and Yo'LLaVA \cite{yollava} inject user-specific concepts through a small number of learnable soft prompts, avoiding long token-level inputs and improving stability. MC-LLaVA \cite{mcllava} further extends this paradigm to multi-concept scenarios. Nevertheless, these methods require retraining for each new concept due to vocabulary expansion. To reduce this overhead, PLVM \cite{plvm} introduces an auxiliary pre-trained visual encoder to add new concepts without inference-time training, but its applicability is restricted to relatively simple task settings. These limitations motivate us to explore more scalable and flexible personalization methods. Meanwhile, studying training data recipes also emerges as a core topic in LVLM personalization \cite{unic, data1, bench}.

\section{Method}
\label{sec:method}

\subsection{Overview} 
\textbf{Problem setting.} Each set $\mathcal{C} = \{C_1, \ldots, C_n\}$ contains $n$ concepts, where each $C_i$ consists of a label and a collection of reference images of varying sizes that exhibit its visual features. The concepts, namely the label–image mappings, are assumed to constitute novel knowledge for a pre-trained LVLM. Thus, given an LVLM $\theta$, the goal of personalization is to enable the model to learn these concepts from scratch and leverage them across a variety of tasks.  Although ICL aligns well with personalization, explicitly encoding all reference images in the prompt incurs high token costs and unstable performance. 

\begin{figure}[tb]
  \centering
  \includegraphics[width=1\linewidth]{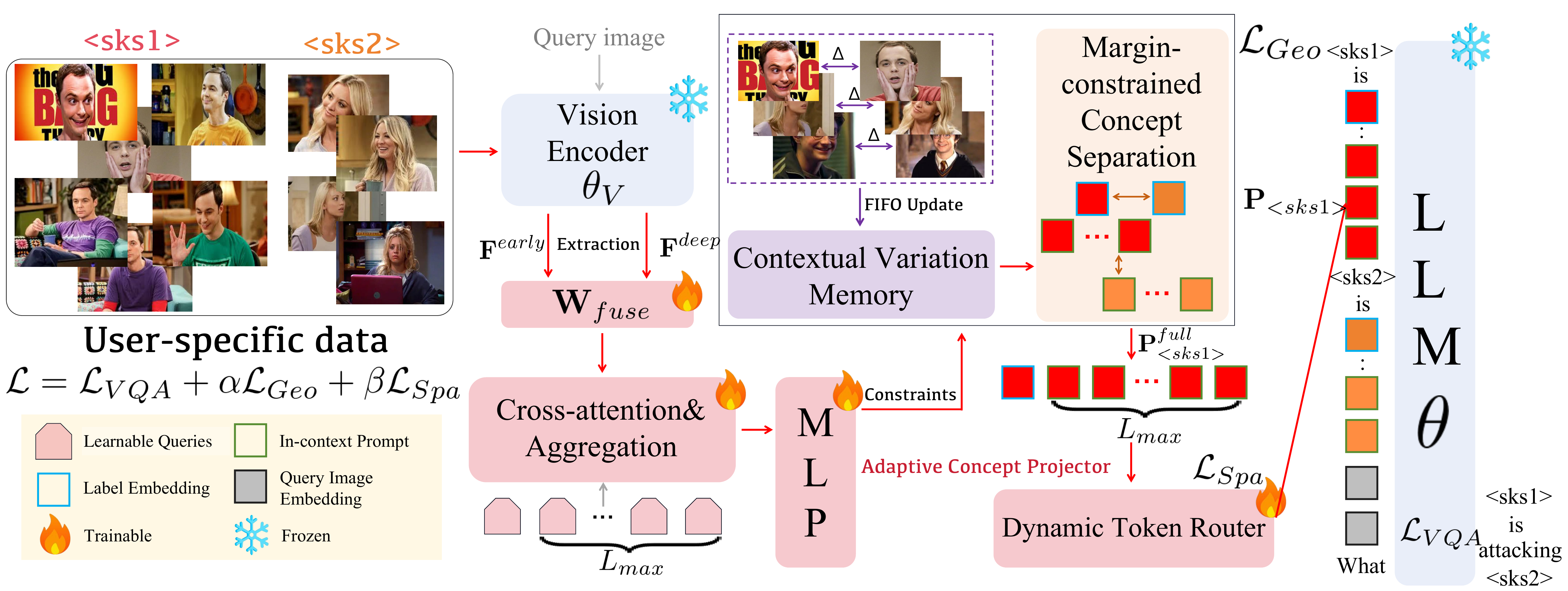}
  \caption{The overall framework of ICPT. The components highlighted in pink constitute the core Adaptive Concept Projector.}
  \label{fig:pipeline}
\end{figure}

Therefore, we propose in-context prompt tuning (ICPT), which simulates ICL in the latent space. Its pipeline is shown in Fig. \ref{fig:pipeline}. We first establish our real-time setting (Sec. \ref{subsec:overview}). Next, we utilize a lightweight Adaptive Concept Projector and its dynamic token allocation mechanism (Sec. \ref{subsec:acp}). We then introduce two strategies to enhance the multimodal semantic capture of in-context prompts: Contextual Variation Memory, which decouples transient environmental states from intrinsic identity (Sec.~\ref{subsec:subspace}), and Margin-constrained Concept Separation, which mitigates cross-concept confusion while preserving shared semantic structure (Sec.~\ref{subsec:margin}). Finally, we describe the training strategy (Sec.~\ref{subsec:training}).

\subsection{Multi-concept In-context Prompts} \label{subsec:overview} 
For each concept $C_i$, an Adaptive Concept Projector (ACP) processes the reference images to generate a continuous in-context prompt $\mathbf{P}_i \in \mathbb{R}^{L_i \times d}$, where $d$ is the hidden dimension of the LVLM. Unlike prior works that force a fixed prompt length for all concepts, the ACP then determines a variable token length $L_i \le L_{max}$ based on the visual complexity of the specific concept. It simultaneously generates a token embedding $\mathbf{w}_i \in \mathbb{R}^{1 \times d}$ to represent the label of each concept. During personalization, to associate a generic label with the visual features of a concept, a two-concept LVLM input can be formatted as follows:
\begin{quote} 
\texttt{[System]: We define} \\ 
\texttt{<sks1> is $\mathbf{w}_1: \mathbf{P}_1$} \\ 
\texttt{<sks2> is $\mathbf{w}_2: \mathbf{P}_2$} \\ 
\texttt{[User]: (Query Image) What is <sks2> doing to <sks1>?} 
\end{quote}

Here, the textual labels (\texttt{<sks1>}, \texttt{<sks2>}) are retained, enabling the frozen LVLM to output them in open-ended responses using its pre-trained vocabulary. $\mathbf{w}_i$ serves as the dedicated semantic anchor for each label, while the subsequent $\mathbf{P}_i$ provides visual evidence. Under this real-time setting, the challenge is ensuring that these in-context prompts and their anchors precisely capture the identities while maintaining disentanglement during multi-concept routing. 

\subsection{Adaptive Concept Projector (ACP)} \label{subsec:acp} In realistic settings, reference images of a concept are often not clean or well-isolated. Instead, they may contain visual distractions such as co-occurring objects and complex backgrounds. To address this issue, existing approaches either employ an auxiliary mask generator to crop concept-specific regions or introduce an additional vision encoder to extract embeddings dedicated to personalization. 

To eliminate such external dependencies while improving efficiency, we directly leverage the built-in vision encoder $\theta_V$ of the LVLM. For each reference image, we feed it into $\theta_V$ and, inspired by \cite{deepstack}, extract visual representations from its early and deep layers. By utilizing these hierarchical embeddings, we capture multi-level visual characteristics of objects within the image, enabling fine-grained concept identification without requiring external supervision. Specifically, let $I_{i,j}$ denote the $j$-th reference image of concept $C_i$, where $j \in \{1, \dots, N_i\}$. We extract its patch-wise embeddings from the designated early and deep layers of $\theta_V$, denoted respectively as $\mathbf{F}_{i,j}^{early}, \mathbf{F}_{i,j}^{deep} \in \mathbb{R}^{S \times d_v}$, where $S$ is the sequence length of the patches and $d_v$ is the feature dimension of the vision encoder. We fuse these multi-scale features via channel-wise concatenation followed by a linear projection weight $\mathbf{W}_{fuse}$: 
\begin{equation} \mathbf{F}_{i,j} = [\mathbf{F}_{i,j}^{early}  \parallel \mathbf{F}_{i,j}^{deep}] \mathbf{W}_{fuse} \in \mathbb{R}^{S \times d}, \end{equation} 
where $\parallel$ denotes the concatenation operator and $\mathbf{W}_{fuse} \in \mathbb{R}^{2d_v \times d}$ projects the fused features into the LVLM's hidden dimension $d$. 

After obtaining the fused visual embeddings for each image, we employ the ACP, which consists of a single-layer cross-attention module followed by an MLP, to transform them into in-context visual prompts and their corresponding textual label embeddings. We initialize a fixed set of $L_{max} + 1$ learnable latent queries $\mathbf{Q} \in \mathbb{R}^{(L_{max} + 1) \times d}$. The ACP distills the dense visual features into a fixed-capacity representation via cross-attention: \begin{equation} \mathbf{v}_{i,j} = \text{Softmax}\left(\frac{(\mathbf{Q}\mathbf{W}_q)(\mathbf{F}_{i,j}\mathbf{W}_k)^\top}{\sqrt{d}}\right) (\mathbf{F}_{i,j}\mathbf{W}_v) \in \mathbb{R}^{(L_{max} + 1) \times d}, \end{equation} 
where $\mathbf{W}_q, \mathbf{W}_k, \mathbf{W}_v \in \mathbb{R}^{d \times d}$ are learnable projection matrices. The visual centroid across the $N_i$ reference images is aggregated as $\bar{\mathbf{v}}_i = \frac{1}{N_i} \sum_{j=1}^{N_i} \mathbf{v}_{i,j}$. This centroid is then processed by the MLP to generate a joint representation $[\mathbf{w}_i ; \mathbf{P}_i^{full}] = \text{MLP}(\bar{\mathbf{v}}_i)$, where the first token $\mathbf{w}_i \in \mathbb{R}^{1 \times d}$ serves as the dedicated word embedding for the text label, and the remaining $L_{max}$ tokens form the intermediate full-capacity visual prompt $\mathbf{P}_i^{full} \in \mathbb{R}^{L_{max} \times d}$. 

As different concepts vary in visual complexity, we aim to allocate prompts dynamically to balance accuracy and efficiency. To this end, the ACP incorporates a Dynamic Token Router (DTR) applied exclusively to the visual prompt. The DTR utilizes a lightweight linear routing head and a Sigmoid activation $\sigma$ to assign a continuous score $s_{i,l}$ to each of the $L_{max}$ visual tokens: 
\begin{equation} s_{i,l} = \sigma(\mathbf{p}_{i,l}^{full} \mathbf{w}_{router} + b_{router}) \in (0, 1), \quad \text{for } l \in \{1, \dots, L_{max}\}, \end{equation} 
where $\mathbf{p}_{i,l}^{full} \in \mathbb{R}^{1 \times d}$ is the $l$-th token of $\mathbf{P}_i^{full}$, $\mathbf{w}_{router} \in \mathbb{R}^{d \times 1}$, and $b_{router}$ is a bias term. During training, to preserve end-to-end differentiability, we apply soft-masking via element-wise multiplication: $\mathbf{P}_i = \mathbf{s}_i \odot \mathbf{P}_i^{full}$, where $\mathbf{s}_i = [s_{i,1}, \dots, s_{i,L_{max}}]^\top$ and $\odot$ denotes row-wise broadcasting multiplication. During inference, we dynamically prune redundant visual tokens by retaining only the $L_i$ tokens satisfying $s_{i,l} > \tau$, guaranteeing $L_i \ge 1$. The discrete label embedding $\mathbf{w}_i$, acting as the global text anchor, is preserved without pruning. This ensures each concept occupies an optimal, minimal footprint in the LVLM context. 

In summary, the ACP defines the prompt-generation process. It takes $C_i$'s hierarchical visual features as input and produces a full-capacity visual prompt $\mathbf{P}_i^{full}$ and an embedding of the concept label $\mathbf{w}_i$. The visual prompt is then dynamically routed into a soft-masked prompt $\mathbf{P}_i$, which is incorporated into the LVLM prompt alongside $\mathbf{w}_i$, introduced in Sec.~\ref{subsec:overview}. However, to ensure the representations capture pure, disentangled concepts, we regularize them during training via two geometric constraints, detailed next. 

\subsection{Contextual Variation Memory} \label{subsec:subspace} Although the generated full-capacity prompt $\mathbf{P}_i^{full}$ encodes the target identity, it also retains biases from environmental factors such as background and lighting. If unaddressed, the LVLM may overfit to these cues rather than the true identity. In multi-concept settings, estimating such noise from intra-concept differences is unreliable because the aggregated centroid is itself biased by visual context.

To isolate and filter this noise without identity leakage, we construct a \textbf{Contextual Variation Memory} (CVM), denoted $\mathcal{M}_{ext}$, dynamically maintained as a First-In-First-Out (FIFO) queue with a fixed capacity limit $K$. During training, we concurrently perform two operations: populating this memory with environmental variations and utilizing it to regularize the prompts of the target concepts. To populate $\mathcal{M}_{ext}$ without introducing significant additional overhead, we reuse the representations already computed during the ACP's forward pass. For each concept $C_i$ in the current batch with $N_i \ge 2$ reference images, we randomly sample one pair of distinct indices $(x, y)$. To extract a sharp environmental transition without the blurring effect of multi-image averaging, we bypass centroid aggregation and independently pass $\mathbf{v}_{i,x}$ and $\mathbf{v}_{i,y}$ through the shared MLP to obtain their single-image full-capacity visual prompts, $\mathbf{P}_{i,x}^{full}$ and $\mathbf{P}_{i,y}^{full}$. In practical personalization scenarios, the provided reference images are typically selected to preserve stable or correlated identity components across different views. As a result, these identity components are expected to largely offset each other in the visual latent space. Thus, their Euclidean difference yields a directional matrix that captures environment-induced variations: 
\begin{equation} \boldsymbol{\Delta}_{i, x, y} = \mathbf{P}_{i,x}^{full} - \mathbf{P}_{i,y}^{full} \in \mathbb{R}^{L_{max} \times d}. \end{equation} 

We then normalize this matrix to obtain $\hat{\boldsymbol{\Delta}} = \boldsymbol{\Delta}_{i, x, y} / \|\boldsymbol{\Delta}_{i, x, y}\|_F$ and push it into $\mathcal{M}_{ext}$. Once the memory queue reaches its capacity limit $K$, the oldest transitions are dequeued as new ones are added. Meanwhile, we utilize this updated memory to regularize the multi-image target concepts being optimized. For the aggregated visual prompt $\mathbf{P}_i^{full}$ of a target concept, we normalize it onto the unit hypersphere ($\hat{\mathbf{P}}_i^{full} = \mathbf{P}_i^{full} / \|\mathbf{P}_i^{full}\|_F$) and explicitly penalize its projection onto the variation matrices currently stored in $\mathcal{M}_{ext}$. Because both operands are Frobenius-normalized, their Frobenius inner product is mathematically equivalent to their cosine similarity. This constraint encourages the prompt to suppress components aligned with dominant contextual variation directions: \begin{equation} \langle \hat{\mathbf{P}}_i^{full}, \hat{\mathbf{d}} \rangle_F = 0, \quad \forall\hat{\mathbf{d}} \in \mathcal{M}_{ext}. \label{eq:ortho_constraint} \end{equation} 

By satisfying this constraint across all discrete matrices in $\mathcal{M}_{ext}$, we geometrically drive the learned prompt into the orthogonal null space of the continuous, identity-agnostic subspace spanned by the memory queue. By equipping the matrix space $\mathbb{R}^{L_{max} \times d}$ with the Frobenius inner product, we formally denote this contextual variation subspace as $\hat{\mathcal{V}}_{ext} = \text{span}(\mathcal{M}_{ext})$. 

This regularization provides an explicit decomposition and bound on state-induced interference. Consider the LVLM evaluating a novel query image at test time. Let $\hat{\mathbf{q}}$ denote the Frobenius-normalized representation of a concept observed under a novel condition, and let $\hat{\mathbf{v}}_{id}$ denote the ideal, identity-dominant reference representation of that same concept. The deviation caused by the novel environment is defined as $\boldsymbol{\delta}_{state} \triangleq \hat{\mathbf{q}} - \hat{\mathbf{v}}_{id}$. We can orthogonally decompose this state-induced shift with respect to our learned subspace $\hat{\mathcal{V}}_{ext}$: \begin{equation} \boldsymbol{\delta}_{state} = \boldsymbol{\delta}_{state}^{\parallel} + \boldsymbol{\delta}_{state}^{\perp}, \quad \boldsymbol{\delta}_{state}^{\parallel} \in \hat{\mathcal{V}}_{ext}, \quad \boldsymbol{\delta}_{state}^{\perp} \perp \hat{\mathcal{V}}_{ext}. \end{equation} 

Because $\hat{\mathcal{V}}_{ext}$ captures a broad spectrum of environmental shifts, $\boldsymbol{\delta}_{state}^{\parallel}$ represents the predictable environmental interference, while $\boldsymbol{\delta}_{state}^{\perp}$ is the unpredictable, out-of-span residual. The following theorem formally proves that enforcing Eq.~\ref{eq:ortho_constraint} during training neutralizes this interference during inference.
\begin{theorem}[State-Induced Interference Bound] \label{thm:stateagnostic} If the learned prompt satisfies the orthogonality constraint $\langle \hat{\mathbf{P}}_i^{full}, \hat{\mathbf{d}} \rangle_F = 0$ for all $\hat{\mathbf{d}} \in \mathcal{M}_{ext}$, then the Frobenius inner product (which is equivalent to cosine similarity for Frobenius-normalized matrices) between the prompt and the novel query decomposes as \begin{equation} \langle \hat{\mathbf{P}}_i^{\,full}, \hat{\mathbf{q}} \rangle_F = \langle \hat{\mathbf{P}}_i^{\,full}, \hat{\mathbf{v}}_{id} \rangle_F + \langle \hat{\mathbf{P}}_i^{\,full}, \boldsymbol{\delta}_{state}^{\perp} \rangle_F, \end{equation} and the contribution of state-induced deviation is bounded by \begin{equation} \left| \langle \hat{\mathbf{P}}_i^{\,full}, \boldsymbol{\delta}_{state}^{\perp} \rangle_F \right| \le \left\| \boldsymbol{\delta}_{state}^{\perp} \right\|_F . \end{equation} \end{theorem} 
\begin{remark}  Theorem~\ref{thm:stateagnostic}'s proof is provided in the Appendix. It establishes that enforcing orthogonality to the subspace $\hat{\mathcal{V}}_{ext}$ attenuates interference aligned with the stored variation directions. As $\mathcal{M}_{ext}$ is continuously updated with diverse, Frobenius-normalized pairwise transitions during training, it provides a progressively richer empirical reference frame for common environment-induced changes. 
\end{remark} 

\subsection{Margin-constrained Concept Separation} \label{subsec:margin} 
When multiple concepts populate the LVLM's visual dictionary simultaneously, cross-concept interference can manifest in two distinct pathways: \textit{visual-space collision} between the candidate visual prompts ($\mathbf{P}_i^{full}$ and $\mathbf{P}_j^{full}$) during cross-attention, and \textit{textual-space collision} between the generated label embeddings ($\mathbf{w}_i$ and $\mathbf{w}_j$) during self-attention. Enforcing global orthogonality between two concepts by driving their similarity toward zero can induce catastrophic feature loss. For example, if two concepts are both specific dog breeds, forcing their embeddings to be orthogonal harms their shared semantic priors, collapsing the representations out-of-distribution. 

To address this issue across modalities, we introduce \textbf{Margin-constrained Concept Separation} (MCS). Rather than driving the cross-concept Frobenius inner product (which equates to cosine similarity for Frobenius-normalized matrices) to zero, we implement this constraint using a soft margin boundary $m \in (0, 1)$ and enforce it dually on both the continuous visual prompts and the discrete label embeddings: \begin{equation} \langle \hat{\mathbf{P}}_i^{full}, \hat{\mathbf{P}}_j^{full} \rangle_F \le m, \quad \text{and} \quad \langle \hat{\mathbf{w}}_i, \hat{\mathbf{w}}_j \rangle \le m, \quad \forall i \neq j, \end{equation}
where $\hat{\mathbf{w}}_i = \mathbf{w}_i / \|\mathbf{w}_i\|_2$ denotes the $L_2$-normalized textual label embedding, and the vector inner product corresponds to their cosine similarity. 

This bounded projection allows the representations to dynamically preserve shared features up to the similarity threshold $m$, while making unique identity embeddings remain separated enough to prevent confusion in both textual and visual spaces. By imposing these pre-DTR constraints, we ensure that the in-context prompts are less affected by visual noise and cross-concept interference.
\setcounter{footnote}{0}
\subsection{Training}
\label{subsec:training}
  \paragraph{Training Data.} To equip ICPT with robust multi-image, multi-concept personalization capabilities, we construct a dedicated, high-quality training dataset due to the lack of suitable existing resources. We first collect 350 concepts from four sources: three existing training sets (MyVLM, Yo'LLaVA, and MC-LLaVA) and additional concepts curated from public media repositories. Each concept is assigned a generic label that is independent of the LVLM's prior knowledge. The selected concepts cover diverse categories, including people, animals, objects, and scenes. For each concept, we prepare a pool of 10 reference images that provide informative visual cues. These images are either directly collected real-world captures or generated using image editing by Nano Banana2\footnote{https://gemini.google/overview/image-generation/}, allowing us to explicitly introduce controlled variations in appearance, context, and background within each concept. During training, the number of reference images sampled per concept ranges from 1 to 6 and is non-uniformly distributed.

Based on these concepts, we construct diverse query images and user queries. In total, we collect 2000 query images, each containing between 0 and 5 concepts. To ensure robust reasoning and explicitly penalize text-prior hallucinations, the data distribution is carefully balanced: single-concept settings account for 35\% of the data, multi-concept cases comprise 60\%. For each query image, we utilize Gemini3.1-Pro\footnote{https://deepmind.google/models/gemini/pro/} to generate 10 candidate user queries spanning a wide range of multimodal tasks, including existence recognition, multiple-choice visual question answering (VQA), and MSCOCO-style \cite{coco} captioning. Finally, we conduct two rounds of manual filtering to retain the five highest-quality user queries per image for training. Detailed statistics are provided in the Appendix.
\paragraph{Training Objective.} To optimize the ACP while keeping the LVLM backbone frozen, we adopt a three-term multi-task objective:
\begin{equation}
    \mathcal{L} = \mathcal{L}_{VQA} + \alpha \mathcal{L}_{Geo} + \beta \mathcal{L}_{Spa}.
\end{equation}

Here, $\mathcal{L}_{VQA}$ represents the standard cross-entropy loss over the generated answer tokens $y_t$ of length $T$. Given the user query $X_{query}$, query image $I_{query}$, and the set of active concepts $\mathcal{C}$, the VQA loss is formulated as:
\begin{equation}
    \mathcal{L}_{VQA} = - \frac{1}{T} \sum_{t=1}^T \log P(y_t | y_{<t}, X_{query}, I_{query}, \{\mathbf{w}_i, \mathbf{P}_i\}_{i \in \mathcal{C}}).
\end{equation}

Meanwhile, $\mathcal{L}_{Geo}$ geometrically unifies the constraints introduced in Sec.~\ref{subsec:subspace} and Sec.~\ref{subsec:margin} into a learning objective. For a batch containing active concepts $\mathcal{C}$ and a sampled subset of normalized variation matrices $\hat{\mathcal{M}}_{b}$, we define:
\begin{equation}
\small
\resizebox{\textwidth}{!}{$
    \mathcal{L}_{Geo} 
    = \sum_{i \in \mathcal{C}} 
    \left( 
    \frac{1}{|\hat{\mathcal{M}}_{b}|} 
    \sum_{\hat{\boldsymbol{\Delta}} \in \hat{\mathcal{M}}_{b}} 
    \langle \hat{\mathbf{P}}_i^{full}, \hat{\boldsymbol{\Delta}} \rangle_F^2 
    + 
    \frac{1}{|\mathcal{C}|-1} 
    \sum_{\substack{j \in \mathcal{C} \\ j \neq i}} 
    \left[ \max\left(0, \langle \hat{\mathbf{P}}_i^{full}, \hat{\mathbf{P}}_j^{full} \rangle_F - m \right)^2 + \max\left(0, \langle \hat{\mathbf{w}}_i, \hat{\mathbf{w}}_j \rangle - m \right)^2 \right] 
    \right).
$}
\end{equation}

Finally, $\mathcal{L}_{Spa}$ applies an $L_1$ penalty to the DTR's predicted scores, actively encouraging the ACP to drop redundant visual tokens and minimize the expected length of in-context prompts:
\begin{equation}
    \mathcal{L}_{Spa} = \frac{1}{|\mathcal{C}| \cdot L_{max}} \sum_{i \in \mathcal{C}} \sum_{l=1}^{L_{max}} s_{i,l}.
\end{equation}

\section{Experiment}
\subsection{Experimental Setup}
\label{sub:4.1}
\paragraph{Evaluation.} Existing LVLM personalization benchmarks generally fail to satisfy the two evaluation requirements of our setting: support for multi-image, multi-concept personalization, and an out-of-distribution test set with respect to the training data. Therefore, to maintain consistency with existing protocols, we adapt data from prior benchmarks and follow their construction principles to enable effective evaluation. We collect concepts from MyVLM, Yo'LLaVA, MC-LLaVA, MMPB \cite{bench}, and additional concepts curated and annotated from public media, resulting in a total of 200 concepts. All selected concepts are verified to be absent from the training set. For each concept, we obtain 1 to 6 reference images through collection or synthesis, with the number of images uniformly distributed across concepts. At test time, all reference images for each concept are used for personalization. These concepts correspond to 300 collected query images, each containing between 0 and 6 concepts, where 6 represents an out-of-distribution setting. We then construct 10 high-quality user queries for each query image by combining existing annotations with those generated by Gemini3.1-Pro. These queries cover the three task types used in training and a novel open-ended VQA task. We report performance separately for each task type under single-concept and multi-concept settings, comprising 1050 and 2100 cases, respectively. Among them, 150 cases correspond to open-ended VQA without any query image, where the LVLM answers solely based on the user query. 

\paragraph{Models and baselines.} We conduct experiments on LLaVA-NeXT-7B \cite{llavanext}, LLaVA-NeXT-34B, InternVL3-8B \cite{intern}, and Qwen3VL-8B \cite{Qwen3}. Results on LLaVA-NeXT-7B serve as the primary comparison. We benchmark against multiple personalization baselines, including ICL, MyVLM \cite{myvlm}, Yo'LLaVA \cite{yollava}, PVIT \cite{pvit}, RAP \cite{rap}, PeKit \cite{pekit}, PLVM \cite{plvm}, and MC-LLaVA \cite{mcllava}. We also include GPT-4o with ICL as a closed-source baseline. On the other three LVLMs, we compare ICPT with ICL and PLVM to demonstrate its generalization across different LVLM backbones. For closed-source baselines, we reproduce their pipelines to the extent possible and report the best scores. For baselines originally designed only for single-concept personalization, we first report their results under the original setting. We then adapt them to the multi-concept settings strictly following their official implementations and report their best scores. If a method cannot be properly adapted, we mark its entries as ``-''. For methods that expand LVLM's vocabulary, we retrain them using \textit{in-distribution} data to ensure a fair comparison. 

\paragraph{Implementation details.} The MLP used in the ACP has three linear layers, where the hidden dimension is set to twice the embedding dimension of the LVLM, and GeLU is used as the activation function. During training, only the parameters of the ACP and its DTR module are updated, while the LVLM backbone remains frozen. We use AdamW as the optimizer with a learning rate of 1e-4 and a batch size of 8. Across different LVLMs, we fix the full capacity of each in-context prompt to $L_{\text{max}} = 20$ and the capacity limit of the CVM to $K = 128$, set the DTR decision threshold to $\tau = 0.5$, and use a soft margin of $m = 0.15$. For visual feature extraction, we consistently take representations from the third layer and the third-to-last layer of the LVLM vision encoder. The coefficients $\alpha$ and $\beta$ are selected in a model-specific manner. For example, on LLaVA-NeXT-7B, we set $\alpha = 0.3$ and $\beta = 0.5$. All experiments are conducted on two NVIDIA H200 GPUs. All additional details of Sec.~\ref{sub:4.1} are provided in the Appendix.
\subsection{Main Results}
\begin{table*}[t]
\centering
\caption{Performance across four personalization task types under two settings, and their weighted average, on LLaVA-NeXT-7B. For existence recognition, we report recall; for multiple-choice VQA (MVQA), accuracy; for open-ended VQA (OVQA), BLEU; and for captioning, we report the average of concept recall in the generated captions and their embedding similarity to the ground-truth captions using DeBERTa-v3-large \cite{deberta}.} 
\small
\resizebox{\textwidth}{!}{
\begin{tabular}{c||ccc|ccc|ccc|ccc}
\toprule
\multirow{2}{*}{\textbf{Method}}
& \multicolumn{3}{c|}{\textbf{Recognition}}
& \multicolumn{3}{c|}{\textbf{MVQA}}
& \multicolumn{3}{c|}{\textbf{OVQA}}
& \multicolumn{3}{c}{\textbf{Captioning}} \\
\cmidrule(lr){2-4}
\cmidrule(lr){5-7}
\cmidrule(lr){8-10}
\cmidrule(lr){11-13}
& Single & Multi & Weight
& Single & Multi & Weight
& Single & Multi & Weight
& Single & Multi & Weight \\
\midrule
GPT-4o + ICL          &0.726&  0.653& 0.702 & 0.695& 0.651 & 0.661 &0.663 & 0.605& 0.619 & 0.674 & 0.608 & 0.630 \\
\midrule
Vanilla                  &0.500&  0.500&0.500 & 0.565 & 0.370 & 0.413 & 0.367 & 0.325 & 0.335 & 0.208 & 0.159 & 0.175\\
ICL                   &0.632&  0.560&0.608 & 0.705 & 0.609 & 0.630 &0.572  & 0.483 & 0.504 &  0.289& 0.137 & 0.188 \\
MyVLM               & 0.718& - &- &0.645  & - & - & 0.548 &  -& - &  0.493&  -&  -\\
Yo’ LLaVA    &0.709& 0.640&0.686  &0.680  & 0.597 & 0.615 & 0.593 & 0.451 & 0.485 & 0.506 & 0.385 & 0.425 \\
PVIT                 & 0.752& 0.633& 0.712 & 0.695 & 0.647 & 0.658 & 0.642 & 0.537 &0.562  & 0.623 &  0.467&0.519  \\
RAP                   &0.725&  0.621& 0.690 &  0.710& 0.627 & 0.645 & 0.651 &  0.529& 0.558 &  0.580&  0.493&  0.522\\
PeKit                 &0.684&  0.647& 0.672 &0.700  & 0.657 & 0.667 &0.680  &  0.524& 0.561 & 0.658 & 0.486 & 0.543 \\
PLVM               &0.794 & 0.719 & 0.769 & 0.760 & 0.673 & 0.692 &  0.695& 0.607 & 0.628 & 0.664& 0.507&  0.559\\
MC-LLaVA             & 0.824& 0.785 & 0.811 &  0.815&  0.709& 0.733& 0.689 & 0.630 & 0.644 & 0.686 &  0.538& 0.587\\
\textbf{ICPT (Ours)}  &\textbf{0.873}&  \textbf{0.859}&  \textbf{0.868}& \textbf{0.850} & \textbf{0.774} &\textbf{0.791}  &  \textbf{0.716}& \textbf{0.697} & \textbf{0.702} & \textbf{0.713} & \textbf{0.624} & \textbf{0.654} \\
\bottomrule
\end{tabular}
}
\label{tab:1}
\end{table*}
\textbf{ICPT demonstrates strong performance across different personalization tasks and settings}. As summarized in Table \ref{tab:1}, ICPT achieves state-of-the-art results across all settings for the four task categories, outperforming existing personalization baselines and GPT-4o with ICL. Compared with directly inserting concepts through ICL, ICPT improves both efficiency and effectiveness while using fewer tokens.
On the existence recognition task, ICPT surpasses the prior SOTA baseline MC-LLaVA by 0.057 in weighted score, while using only \textbf{13.6} tokens per concept on average, compared with the fixed 16 tokens used by MC-LLaVA. ICPT also shows larger gains in more complex scenarios. For example, on open-ended VQA, it improves over MC-LLaVA by 0.027 in the single-image setting and by 0.067 in the multi-image setting.
These results confirm that our design effectively addresses the challenges of multi-image and multi-concept personalization. In addition, the strong performance on open-ended VQA and captioning indicates that introducing in-context prompts does not impair the generative ability of the LVLM, allowing ICPT to integrate naturally into existing LVLM-based systems. Qualitative examples are shown in Fig.~\ref{fig:exa}.

\textbf{ICPT exhibits promising generalization across different LVLM sizes and backbones}. Table \ref{tab:2} summarizes the performance of ICPT on three additional LVLMs, where it consistently achieves SOTA results across all settings. Results on LLaVA-NeXT-7B and LLaVA-NeXT-34B show that ICPT adapts well across model scales when $L_{\text{max}}$ is properly configured. ICPT also performs strongly on InternVL3 and Qwen3VL, two recent LVLM architectures, demonstrating its robustness across different backbones. We further observe that ICPT continues to improve personalization performance even when the underlying LVLM already has strong ICL ability, which PLVM fails to maintain. On the two latest LVLMs, InternVL3 and Qwen3VL, PLVM performs worse than direct ICL on several tasks, whereas ICPT consistently provides stable gains. These results indicate that ICPT can effectively leverage advances in foundation LVLMs and highlight its practical value for building next-generation applications.

\begin{table*}[t]
\centering
\caption{
Performance across four personalization task types under two settings when the personalization methods are applied to three additional LVLM backbones.
}
\small
\resizebox{\textwidth}{!}{
\begin{tabular}{c|c||ccc|ccc|ccc|ccc}
\toprule
\multirow{2}{*}{\textbf{LVLM}} & 
\multirow{2}{*}{\textbf{Method}}
& \multicolumn{3}{c|}{\textbf{Recognition}}
& \multicolumn{3}{c|}{\textbf{MVQA}}
& \multicolumn{3}{c|}{\textbf{OVQA}}
& \multicolumn{3}{c}{\textbf{Captioning}} \\
\cmidrule(lr){3-5}
\cmidrule(lr){6-8}
\cmidrule(lr){9-11}
\cmidrule(lr){12-14}
& 
& Single & Multi & Weight
& Single & Multi & Weight
& Single & Multi & Weight
& Single & Multi & Weight \\
\midrule

\multirow{4}{*}{LLaVA-NeXT-34B}
& Vanilla  &  0.500& 0.500 &  0.500&0.585&0.407 &0.447& 0.392 & 0.316 & 0.334 &0.231 & 0.147 & 0.175 \\
& ICL      & 0.695 & 0.584 & 0.658 & 0.715 & 0.646 & 0.661 & 0.636 &0.573  & 0.588 & 0.328 &  0.215& 0.253 \\
& PLVM     & 0.812 & 0.721 & 0.782 & 0.780 & 0.682 & 0.704 & 0.723 & 0.668 & 0.681 & 0.740 & 0.545 & 0.610 \\
& \textbf{ICPT} & \textbf{0.913} & \textbf{0.885} & \textbf{0.904} & \textbf{0.875} &\textbf{0.807}  & \textbf{0.822} & \textbf{0.746} & \textbf{0.715} & \textbf{0.722} & \textbf{0.762}&  \textbf{0.614}& \textbf{0.663} \\
\midrule

\multirow{4}{*}{InternVL3-8B}
& Vanilla  & 0.500 & 0.500 & 0.500 & 0.635 & 0.507& 0.535 & 0.386 &0.379 & 0.381 & 0.235 &0.263& 0.254 \\
& ICL      & 0.738 & 0.705 & 0.727 & 0.825 & 0.702 &0.729  & 0.726 & 0.618 & 0.644 & 0.575 & 0.538 & 0.550 \\
& PLVM     & 0.795 & 0.730 & 0.773 & 0.820 & 0.695 & 0.723&  0.734& 0.597 & 0.630 & 0.806 & 0.762 & 0.777 \\
& \textbf{ICPT} & \textbf{0.945} & \textbf{0.906} & \textbf{0.932} & \textbf{0.905} &\textbf{0.893} & \textbf{0.896} & \textbf{0.740} & \textbf{0.706} & \textbf{0.714} & \textbf{0.828} & \textbf{0.804} &\textbf{0.812}  \\
\midrule

\multirow{4}{*}{Qwen3VL-8B}
& Vanilla  & 0.500 & 0.500 & 0.500 &0.600&0.604&0.603 & 0.403 & 0.364 & 0.373 & 0.306 & 0.257 &  0.273\\
& ICL      &  0.775& 0.743 & 0.764 & 0.865 & 0.826 & 0.835 &0.758  & 0.642 & 0.670 & 0.635 &  0.572& 0.593\\
& PLVM     & 0.820 & 0.735 & 0.792 & 0.825 & 0.799 & 0.805 &  0.760& 0.668 & 0.690 &  0.784& 0.708 & 0.733 \\
& \textbf{ICPT} & \textbf{0.963} & \textbf{0.924} &\textbf{0.950}  & \textbf{0.920} &\textbf{0.867}  & \textbf{0.879} &  \textbf{0.776}& \textbf{0.724} & \textbf{0.736} & \textbf{0.835} & \textbf{0.821} & \textbf{0.826} \\
\bottomrule
\end{tabular}}
\label{tab:2}
\end{table*}

\begin{wraptable}{r}{0.60\textwidth}
\centering
\small
\caption{ICPT's performance on four standard multimodal personalization benchmarks.}
\label{tb1}
\begin{tabular}{ccccc}
\hline
& Yo'LLaVA & MyVLM & MC-LLaVA & MMPB \\
\hline
ICL  & 0.805 & 0.640 & 0.493 & 0.589 \\
MC-L & 0.921 & 0.964 & 0.904 & 0.824 \\
PLVM & 0.875 & 0.905 & 0.857 & 0.814 \\
ICPT & \textbf{0.926} & \textbf{0.972} & \textbf{0.915} & \textbf{0.853} \\
\hline
\end{tabular}
\end{wraptable}
Beyond the above multi-image, multi-concept evaluation, we further evaluate ICPT on four standard public benchmarks to validate its effectiveness. Table~\ref{tb1} reports the average results across models on the YoLLaVA, MyVLM, MC-LLaVA, and MMPB. For the first three benchmarks, all concepts appearing in the ICPT training set are excluded from evaluation to ensure a fair comparison, as these benchmarks assume vocabulary expansion-based personalization. ICPT consistently achieves the best performance across all four benchmarks, demonstrating that its superiority extends beyond our customized evaluation setting.

\begin{figure}[t]
  \centering
  \includegraphics[width=1\linewidth]{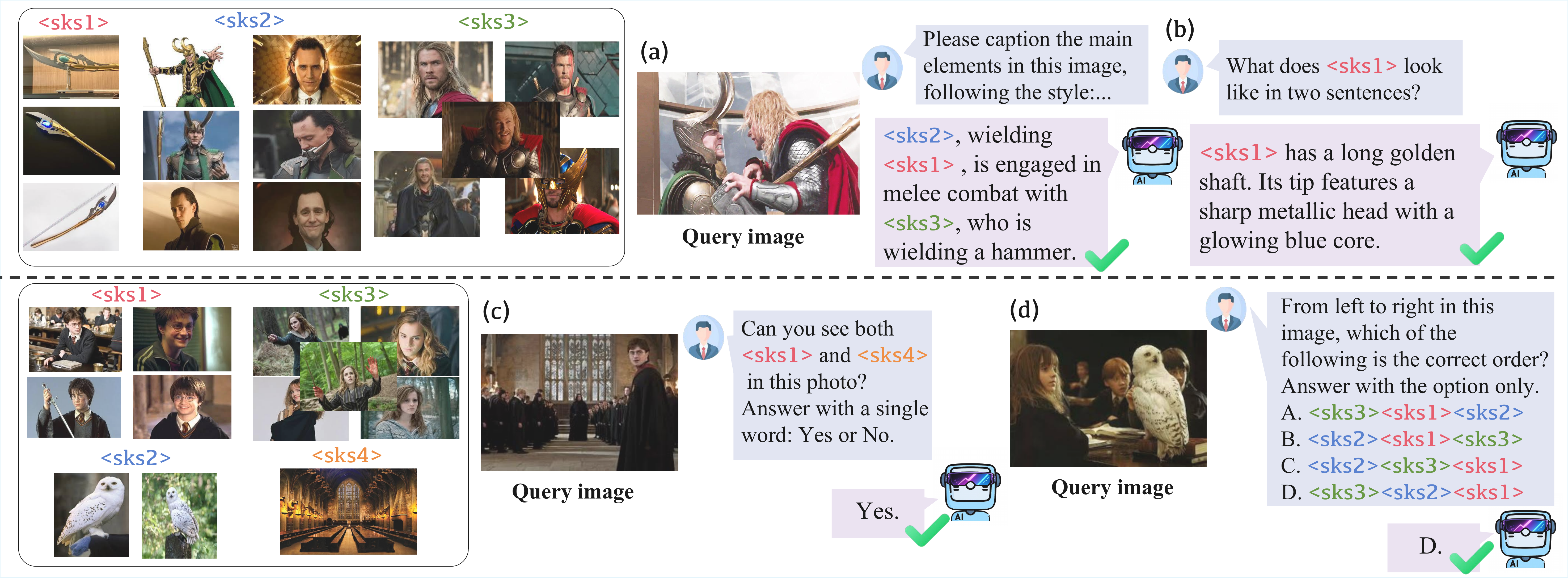}
  \caption{Qualitative examples of ICPT on four tasks: (a) captioning, (b) open-ended VQA without a query image, (c) existence recognition, and (d) multiple VQA. More examples and failure case analyses are provided in the Appendix.}
  \label{fig:exa}
\end{figure}

\begin{figure}[t]
\centering
\begin{minipage}{0.32\textwidth}
\centering
\includegraphics[width=\linewidth]{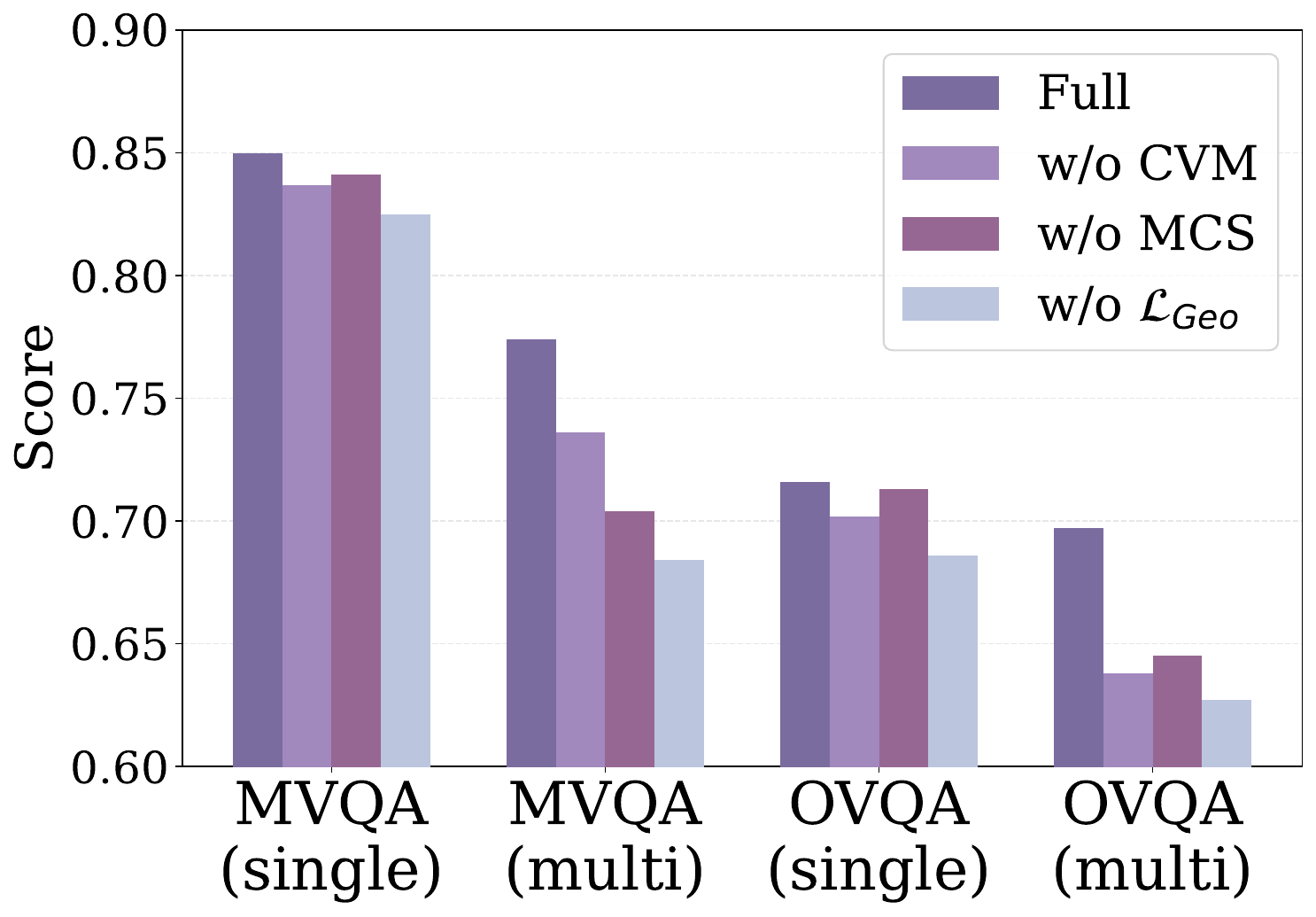}
\captionof{figure}{Ablation study on the proposed constraints.}
\label{fig:fig4}
\end{minipage}
\hfill
\begin{minipage}{0.32\textwidth}
\centering
\includegraphics[width=\linewidth]{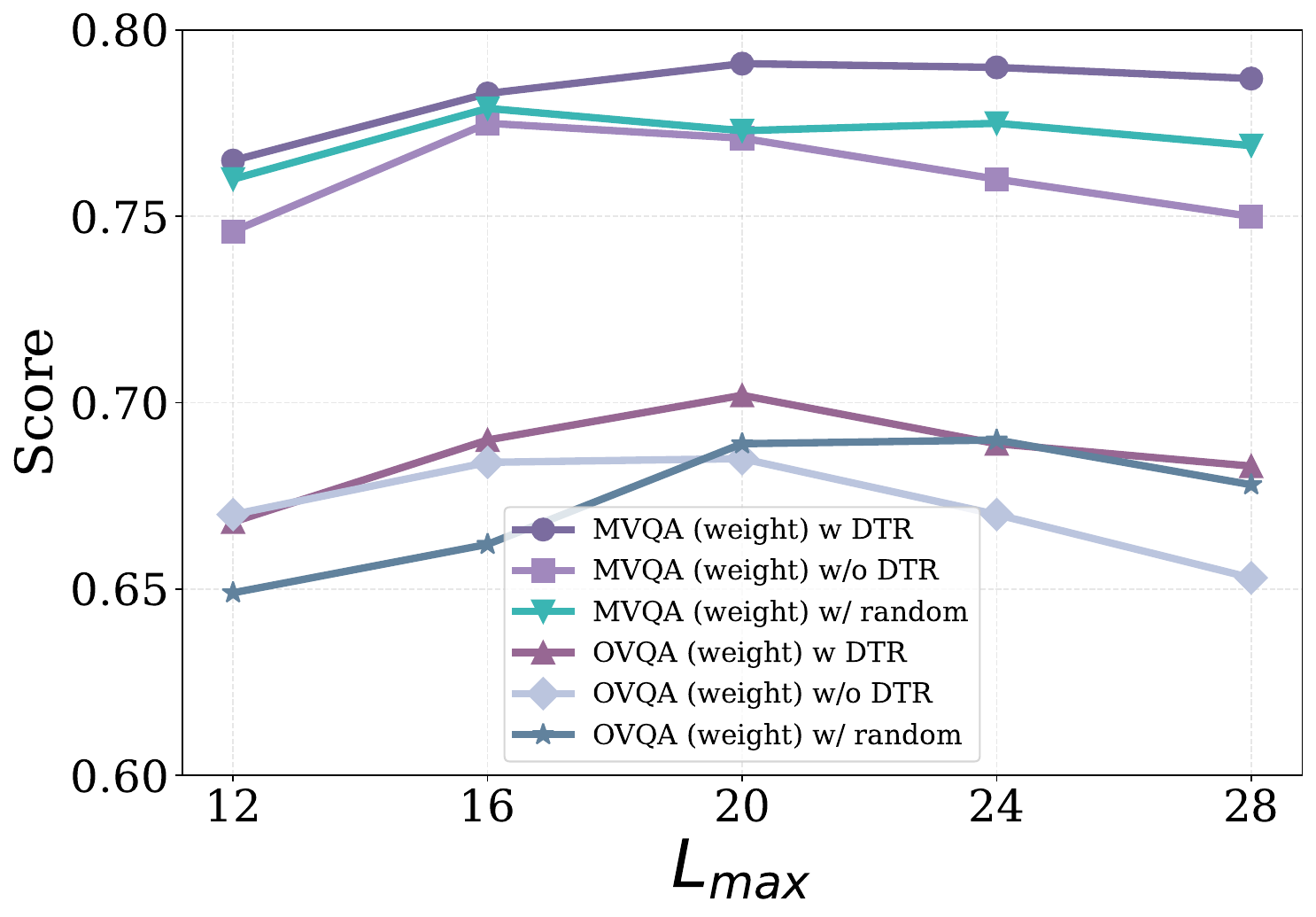}
\captionof{figure}{Ablation study on the DTR mechanism.}
\label{fig:fig5}
\end{minipage}
\hfill
\begin{minipage}{0.32\textwidth}
\centering
\includegraphics[width=\linewidth]{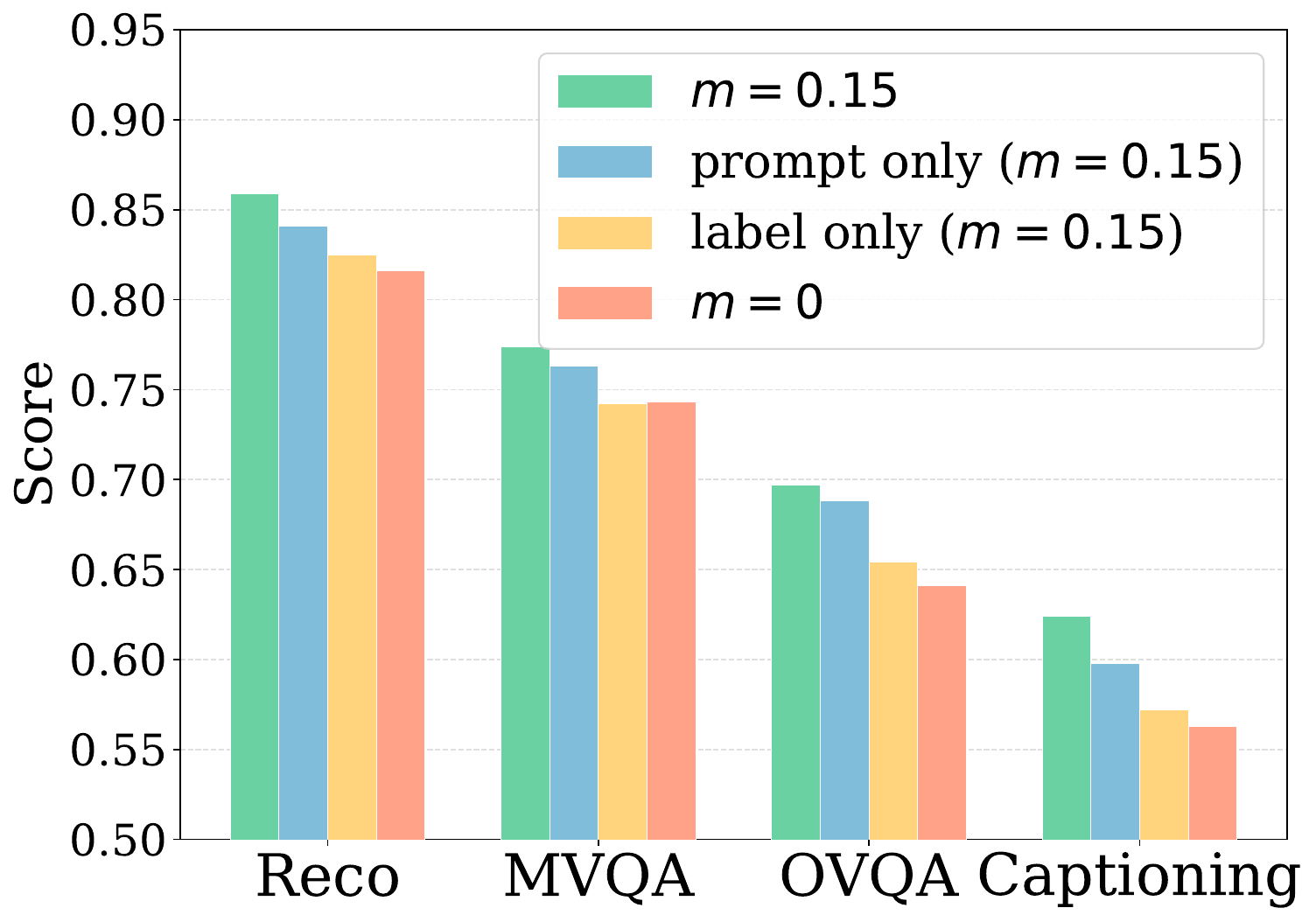}
\captionof{figure}{Ablation study on the MCS setting.}
\label{fig:fig6}
\end{minipage}
\end{figure}

\subsection{Ablation Studies}
We first examine the effectiveness of the two proposed prompt constraints and their corresponding learning objective. The results in Fig.~\ref{fig:fig4} show that both CVM and MCS, together with $\mathcal{L}_{Geo}$, improve performance under both single-concept and multi-concept settings, with larger gains in the multi-concept case. This indicates that the proposed mechanisms are well optimized during training and effectively enhance the semantic quality of the in-context prompts.

We then analyze the impact of the prompt full capacity $L_{\text{max}}$ and the DTR mechanism. One ablation replaces DTR with random pruning at the same ratio. As shown in Fig.~\ref{fig:fig5}, personalization performance generally improves as the prompt length increases, since longer prompts can encode richer visual features. The improvement continues until a threshold. Within this range, the performance gain brought by DTR also becomes more pronounced. Beyond this threshold, however, longer prompts tend to encode redundant features or noise, which leads to performance degradation. In this regime, dynamic pruning becomes even more important. These results indicate that both proper prompt length selection and adaptive pruning are critical for soft prompt-based methods.

For MCS, we design three ablation settings. We apply MCS only to label embeddings or only to full-capacity prompts while keeping the similarity threshold at $m = 0.15$, and we further test the case where $m = 0$, which enforces strict orthogonality. Results in Fig.~\ref{fig:fig6} show that applying MCS only to label embeddings causes a larger performance drop than applying it only to prompts. This suggests that multi-concept confusion mainly arises from overlap in the visual representation space. Nevertheless, applying constraints to both modalities remains important because LVLMs have limited ability to clearly separate them. When $m = 0$, the model overemphasizes separation and weakens the beneficial role of prior knowledge during personalization. This result highlights the importance of selecting an appropriate margin $m$. We provide more comprehensive ablation studies in the Appendix, covering the choice of vision encoder, the selection of feature extraction layers, the coefficients of each loss term, the DTR decision threshold $\tau$, and the CVM capacity $K$.

\subsection{Analyses}
We conduct efficiency analyses of ICPT. With the prompt full capacity set to 20, ICPT achieves 12\% and 18\% lower inference latency on the full test set compared with PLVM and MC-LLaVA, which use fixed 16-token prompts for each concept. This improvement results from using the model’s native vision encoder, together with the lightweight design of ACP and the dynamic pruning enabled by DTR.

We also investigate an open question in LVLM personalization: whether training data volume or diversity is more important. We characterize training data diversity along three dimensions: the visual distribution of reference images per concept, the distribution of the number of concepts per input, and the variety of task types. The results show that a low-volume, high-diversity training setup substantially outperforms its high-volume counterpart. This indicates that exposing the model to complex visual contexts, varied concept combinations, and diverse task formats is far more important than simply increasing the number of collected concepts. As LVLM capabilities advance, effective personalization should move beyond brute-force data memorization and instead serve as a generalizable semantic bridge that aligns concepts with the LVLM’s reasoning in the latent space. These analyses are presented in detail in the Appendix.
\section{Conclusion}
We propose in-context prompt tuning (ICPT), a novel and effective framework that advances the capability of LVLMs to handle complex, multi-concept personalization tasks. By fully shifting the personalization process into the latent space through simulated ICL, our method circumvents the heavy computational burden of inference-time training, allowing users to introduce multiple new concepts on the fly. Central to this capability is the Adaptive Concept Projector (ACP), which efficiently distills reference images into dynamic-length visual prompts and dedicated semantic anchors. Furthermore, to ensure these representations remain robust against background noise and cross-concept entanglement, we design two geometric constraints: Contextual Variation Memory and Margin-constrained Concept Separation. These designs are validated through extensive experiments. We envision that ICPT will provide a strong and practical baseline for future research on LVLM personalization.
\bibliographystyle{splncs04}
\bibliography{main}
\appendix
\section{Personalization Data}
\subsection{Training data}
When constructing the multi-image, multi-concept personalization dataset, we follow three steps: selecting concepts, assigning reference images, and constructing user queries and query images. For the training data, the first source of concepts comes from existing personalization datasets, including MyVLM \cite{myvlm}, Yo'LLaVA \cite{yollava}, and MC-LLaVA \cite{mcllava}. The Yo'LLaVA dataset contains 40 diverse concepts, each represented by 4 to 10 images. MyVLM has 29 object-centric concepts, providing at least 10 images for each concept. MC-LLaVA provides high-quality multi-concept data, covering 50 scenarios with 2 to 4 concepts and a total of 2020 samples. Beyond these datasets, we further collect concepts from films, TV series, animation, and variety shows, as well as public image repositories. We also expand the concept pool through image editing based on these sources to avoid overrepresentation of well-known entities. In total, we obtain 350 concepts covering diverse categories, including characters with different visual styles, animals, objects, buildings, and scenes.
Each concept is assigned a generic label. We ensure that the label does not correlate with the LVLM’s parametric knowledge of the concept, so that the task truly reflects personalization. For example, for an image of Donald Trump, we avoid labels such as <Donald> or <President> and instead use generic identifiers such as <Mike> or <sks1>.

Next, for each concept, we prepare 10 reference images. We impose the following requirements for each set of reference images. Each image must clearly contain part or all of the visual characteristics of the concept so that it can be distinguished from other concepts. Within a set, no two images should be overly similar, and the images should maintain diversity in factors such as entity position, viewpoint, pose, appearance, and background. Among all reference images, 75\% are collected from existing visual materials, while the remaining 25\% are generated based on other reference images of the same concept according to the requirements. During training, the number of reference images sampled for each concept ranges from 1 to 6 and is non-uniformly distributed.

By combining each label with its corresponding reference images, we obtain a complete set of concepts. Based on these concepts, we further construct query images and user queries. In total, we collect 2000 query images, each containing between 0 and 5 concepts. To encourage robust reasoning and reduce text-prior hallucinations, the data distribution is carefully balanced. Single-concept cases account for 35\% of the data, multi-concept cases account for 60\%, and the remaining cases contain zero concepts.
Following prior work, the training data adopts standard dialogue formats. The user queries cover three task types: existence recognition, multiple-choice visual question answering (VQA), and captioning. Existence recognition is the most representative task in LVLM personalization. It requires the model to determine whether a concept appears in the provided query image. A typical user query follows the format “Can you see <sks1> in this photo? Answer with a single word: Yes or No.” Query images that contain zero concepts are also used in this task. These samples naturally serve as negative examples during training, which are important for effective personalization.
Multiple-choice VQA provides the reference images while asking the model about attributes of the query image, together with four answer options labeled A, B, C, and D. The model must directly select the correct option. Captioning requires the model to generate a descriptive caption for the provided query image. After personalization, the model is expected to produce a complete and accurate caption while using the assigned labels to represent the corresponding concepts. All captions follow the MSCOCO style, and each user query includes a textual caption example as guidance. During dataset construction, we use the Gemini3.1-Pro API for text generation and the Nano Banana 2 API for image generation. For each QA pair, we employ four expert annotators to conduct two rounds of review. Problematic questions and incorrect annotations are corrected during this process to ensure the quality of the training data. Only the five highest-quality user queries are retained for each query image, while ensuring diversity across task types. This results in a total of 10000 QA pairs. Among them, 500 correspond to zero-concept cases, 3500 to single-concept cases, 2000 to two-concept cases, 2000 to three-concept cases, 1500 to four-concept cases, and 500 to five-concept cases.  
\subsection{Test data}
The tools and procedures used to construct the test data follow the same pipeline as those used for the training data. When collecting concepts, in addition to MyVLM, Yo'LLaVA, MC-LLaVA, and public image sources, we also include concepts from MMPB \cite{bench}, a benchmark specifically designed for evaluating personalization. In total, we obtain 200 concepts for testing and assign each a generic label. Since our method emphasizes avoiding inference-time training, all concepts in the test set are ensured to be out-of-distribution relative to the training data.
For each concept, we obtain 1 to 6 reference images through collection or synthesis, with the number of images uniformly distributed across concepts. Among them, 85\% are directly collected from existing visual materials and 15\% are synthesized. At test time, all available reference images of a concept are used.
Based on these concepts, we collect 300 query images. Each query image contains between 0 and 6 concepts. We then construct 20 user queries for each query image by combining existing annotations with additional queries generated by the Gemini3.1-Pro API. These queries cover the three task types used in training and a novel open-ended VQA task, which requires the LVLM to answer a question directly based on multimodal reasoning without predefined options. In the open-ended VQA setting, some questions are asked without a query image and require the model to answer solely based on its understanding of the concepts, such as describing the appearance of a concept. Next, we filter the 20 queries for each image and retain 10 high-quality queries to ensure question quality and task diversity with the GPT-5 API. Finally, we manually inspect the resulting queries to remove invalid questions, incorrect gold labels, and overly redundant task types.
Across the four task types, the evaluation set contains 1050 single-concept cases and 2100 multi-concept cases.
\section{Theorem}
\addtocounter{theorem}{-1}
\begin{theorem}[State-Induced Interference Bound] 
If the learned prompt satisfies the orthogonality constraint $\langle \hat{\mathbf{P}}_i^{full}, \hat{\mathbf{d}} \rangle_F = 0$ for all $\hat{\mathbf{d}} \in \mathcal{M}_{ext}$, then the Frobenius inner product (which is equivalent to cosine similarity for Frobenius-normalized matrices) between the prompt and the novel query decomposes as \begin{equation} \langle \hat{\mathbf{P}}_i^{\,full}, \hat{\mathbf{q}} \rangle_F = \langle \hat{\mathbf{P}}_i^{\,full}, \hat{\mathbf{v}}_{id} \rangle_F + \langle \hat{\mathbf{P}}_i^{\,full}, \boldsymbol{\delta}_{state}^{\perp} \rangle_F, \end{equation} and the contribution of state-induced deviation is bounded by \begin{equation} \left| \langle \hat{\mathbf{P}}_i^{\,full}, \boldsymbol{\delta}_{state}^{\perp} \rangle_F \right| \le \left\| \boldsymbol{\delta}_{state}^{\perp} \right\|_F . \end{equation} \end{theorem} 

\begin{proof} The constraint encourages the prompt to suppress components aligned with dominant contextual variation directions: \begin{equation} \langle \hat{\mathbf{P}}_i^{full}, \hat{\mathbf{d}} \rangle_F = 0, \quad \forall\hat{\mathbf{d}} \in \mathcal{M}_{ext}.  \end{equation} 

The Frobenius inner product between the prompt $\hat{\mathbf{P}}_i^{full}$ and query $\hat{\mathbf{q}}$ expands algebraically as: \begin{equation} \small \langle \hat{\mathbf{P}}_i^{full}, \hat{\mathbf{q}} \rangle_F = \langle \hat{\mathbf{P}}_i^{full}, \hat{\mathbf{v}}_{id} + \boldsymbol{\delta}_{state} \rangle_F = \langle \hat{\mathbf{P}}_i^{full}, \hat{\mathbf{v}}_{id} \rangle_F + \langle \hat{\mathbf{P}}_i^{full}, \boldsymbol{\delta}_{state}^{\parallel} \rangle_F + \langle \hat{\mathbf{P}}_i^{full}, \boldsymbol{\delta}_{state}^{\perp} \rangle_F . \end{equation} 

By explicitly constraining the prompt to be mathematically orthogonal to the contextual variation span in Eq.~\ref{eq:ortho_constraint}, the parallel interference term is nullified: $\langle \hat{\mathbf{P}}_i^{full}, \boldsymbol{\delta}_{state}^{\parallel} \rangle_F = 0$. The interaction simplifies to the pure identity match plus the out-of-span interference: $\langle \hat{\mathbf{P}}_i^{full}, \hat{\mathbf{v}}_{id} \rangle_F + \langle \hat{\mathbf{P}}_i^{full}, \boldsymbol{\delta}_{state}^{\perp} \rangle_F$. Applying the Cauchy-Schwarz inequality for the Frobenius inner product, the residual interference error $E$ is bounded by $E = |\langle \hat{\mathbf{P}}_i^{full}, \boldsymbol{\delta}_{state}^{\perp} \rangle_F| \le \|\hat{\mathbf{P}}_i^{full}\|_F \|\boldsymbol{\delta}_{state}^{\perp}\|_F$. Since $\hat{\mathbf{P}}_i^{full}$ is Frobenius-normalized ($\|\hat{\mathbf{P}}_i^{full}\|_F = 1$), the interference error is bounded exactly by $E \le \|\boldsymbol{\delta}_{state}^{\perp}\|_F$. \end{proof} 
\addtocounter{remark}{-1}
\begin{remark}  Theorem~\ref{thm:stateagnostic}'s proof is provided in the Appendix. It establishes that enforcing orthogonality to the subspace $\hat{\mathcal{V}}_{ext}$ attenuates interference aligned with the stored variation directions. As $\mathcal{M}_{ext}$ is continuously updated with diverse, Frobenius-normalized pairwise transitions during training, it provides a progressively richer empirical reference frame for common environment-induced changes. As the memory accumulates more diverse variation directions up to its capacity limit $K$, the learned subspace can better approximate commonly observed contextual shifts, reducing the magnitude of predictable residual components. \end{remark} 
\section{Implementation Details}
We first introduce the used baselines:
\begin{itemize}
    \item \textbf{MyVLM} \cite{myvlm}: MyVLM personalizes a pretrained LVLM with external concept heads that detect the presence of concepts in a set of images using features extracted from its vision encoder. When a concept is detected, a learned concept embedding in the intermediate feature space is injected into the model to condition the LLM decoder. It is designed for multi-image, single-concept settings and cannot be properly extended to multi-concept scenarios.
    \item \textbf{Yo'LLaVA} \cite{yollava}: Yo'LLaVA learns a personalized soft prompt for each concept from a small set of reference images by introducing a new identifier token together with several learnable latent tokens that encode the concept’s visual attributes. It considers both single-image and multi-image settings. In addition, by fusing the tokens assigned to each concept with extended classification head parameters, we adapt it to multi-concept settings, which is a commonly used operation also used in \cite{mcllava}.
    \item \textbf{PVIT} \cite{pvit}: PVIT performs personalized instruction tuning by constructing training samples that combine reference images, query images, and dialogue-style prompts containing concept identifiers. The LVLM is then instruction-tuned on these multimodal conversations so that the model learns to associate the concept labels with their visual identities and generate responses conditioned on the input image and prompt. Since it improves the base LVLM’s ICL ability through instruction tuning, it can be directly applied to multi-image, multi-concept settings. Following the original paper, we use the same training data and training procedure for each LVLM and tune the hyperparameters to achieve the best performance.
    \item \textbf{RAP} \cite{rap}: RAP adopts a retrieval-augmented personalization pipeline consisting of three stages. First, concepts and their associated images or attributes are stored in a key–value database; during inference, a multimodal retriever identifies relevant concepts for the query image; the retrieved concept information is then concatenated with the input prompt and image and fed into the LVLM to generate the final response. It can be directly applied to multi-image, multi-concept settings and updates the concept database online, as in the original architecture.
    \item \textbf{PeKit} \cite{pekit}: PeKit encodes patch-level features from reference images into a persistent memory module $\mathcal{M}$, which are dynamically retrieved by a retrieval module $\mathcal{R}$ and used as visual prompts for personalized inference. As no publicly available source code is provided, we reproduce the method based on the pipeline described in the original paper and ensure that it achieves the reported performance on its original benchmarks. Since it is an ICL-based method, it can be directly applied to multi-image, multi-concept settings.
    \item \textbf{PLVM} \cite{plvm}: PLVM first encodes a reference image of a personalized concept using a pretrained vision encoder and transforms the extracted features into a concept representation composed of a word embedding and a set of context tokens. The resulting concept tokens are inserted as soft prompts. As no publicly available source code is provided, we reproduce the method based on the pipeline described in the original paper and ensure that it achieves the reported performance on its original benchmarks. It considers only single-image, single-concept settings, but the method can be directly extended to multi-concept scenarios. To adapt it to multi-image settings, we adopt the same aggregation strategy as ICPT to combine visual features from multiple reference images of the same concept.
    \item \textbf{MC-LLaVA} \cite{mcllava}: MC-LLaVA first encodes images of multiple user-provided concepts with the vision encoder and projection layer to obtain visual tokens, which are clustered and used to initialize learnable concept tokens in personalized textual prompts. During inference, the model aggregates similarity-based location confidence maps between stored concept features and the test image to construct a personalized visual prompt, which is combined with the textual prompts and input image for the LVLM to generate responses. It is designed for multi-image, multi-concept settings.
\end{itemize}

For baselines such as MyVLM, Yo'LLaVA, and MC-LLaVA, which expand the LVLM vocabulary and therefore support only in-distribution personalization, we construct corresponding training data based on the test concepts so that these methods can operate effectively. All experiments are conducted on the same test data and settings to ensure a fair comparison. 

The four LVLMs used in our experiments are based on the officially released checkpoints, and inference is conducted using the Transformers and HuggingFace libraries. For most hyperparameters in our method, we use the same configuration across all LVLMs, except for the loss coefficients $\alpha$ and $\beta$. On LLaVA-NeXT-7B, we set $\alpha = 0.3$ and $\beta = 0.5$. On LLaVA-NeXT-34B, we set $\alpha = 0.6$ and $\beta = 0.5$. On InternVL3-8B, we set $\alpha = 0.5$ and $\beta = 0.4$. On Qwen3VL-8B, we set $\alpha = 0.6$ and $\beta = 0.4$. For most hyperparameters, we provide ablation studies to examine the sensitivity of ICPT to their choices.

\section{Ablation Studies}
\begin{table*}[t]
\centering
\caption{
Ablation study on the visual feature extraction strategy used as input to ACP.
}
\vspace{-5pt}
\small
\resizebox{\textwidth}{!}{
\begin{tabular}{c||ccc|ccc|ccc|ccc}
\toprule
\multirow{2}{*}{\textbf{Method}}
& \multicolumn{3}{c|}{\textbf{Recognition}}
& \multicolumn{3}{c|}{\textbf{MVQA}}
& \multicolumn{3}{c|}{\textbf{OVQA}}
& \multicolumn{3}{c}{\textbf{Captioning}} \\
\cmidrule(lr){2-4}
\cmidrule(lr){5-7}
\cmidrule(lr){8-10}
\cmidrule(lr){11-13}
& Single & Multi & Weight
& Single & Multi & Weight
& Single & Multi & Weight
& Single & Multi & Weight \\
\midrule

Original 
& 0.873 & 0.859 & 0.868
& 0.850 & 0.774 & 0.791
& 0.716 & 0.697 & 0.702
& 0.713 & 0.624 & 0.654\\
\cmidrule{1-13}
CLIP
& 0.820 & 0.779 & 0.806
& 0.795 & 0.673 & 0.700
& 0.643 & 0.618 & 0.624
& 0.625 & 0.527 & 0.560\\

DINO-v2
& 0.842 & 0.791 & 0.825
& 0.810 & 0.718 & 0.738
& 0.669 & 0.625 & 0.635
& 0.643 & 0.543 & 0.576\\

DINO-v3
& 0.856 & 0.825 & 0.846
& 0.825 & 0.725 & 0.747
& 0.678 & 0.620 & 0.634
& 0.651 & 0.549 & 0.583 \\

\cmidrule{1-13}

Early only
& 0.838 & 0.796 & 0.824
& 0.820 & 0.733 & 0.752
& 0.681 & 0.645 & 0.654
& 0.657 & 0.558 & 0.591 \\

Late only
& 0.832 & 0.787 & 0.817
& 0.830 & 0.748 & 0.766
& 0.667 & 0.628 & 0.637
& 0.641 & 0.532 & 0.568 \\

First both
& 0.859 & 0.823 & 0.847
& 0.825 & 0.756 & 0.771
& 0.705 & 0.674 & 0.681
& 0.685 & 0.583 & 0.617\\

Fifth both
& 0.865 & 0.831 & 0.854
& 0.835 & 0.767 & 0.782
& 0.697 & 0.669 & 0.676
& 0.702 & 0.608 & 0.639\\

\bottomrule
\end{tabular}}
\label{app:ab1}
\end{table*}
\paragraph{Adaptive Concept Projector (ACP).} For the ACP, in addition to designing strategies that optimize the transformation from visual features to in-context prompts, it is also crucial to ensure that the input features themselves are of high quality. Some baselines, such as MC-LLaVA, introduce auxiliary modules like GroundedSAM \cite{grounded} to generate masks for the main characters and improve feature quality. PLVM instead employs an additional vision encoder, though the reference images used in its setting are relatively simple. In ICPT, we take a different approach by directly using the model’s own vision encoder and fusing cross-layer representations. This design allows the model to capture cross-image and fine-grained visual information without introducing additional heavy modules. To investigate this component, we replace the original visual feature extractor with CLIP, DINO-v2 \cite{dinov2}, and DINO-v3 \cite{dinov3}. We also perform ablations on the feature extraction layers while keeping the vision encoder unchanged, replacing the default third and third-to-last layers with either one of them alone, or with the first and fifth layers from both ends. The results are reported in Table~\ref{app:ab1}. We observe that replacing the LVLM’s native vision encoder with other pretrained vision encoders consistently leads to performance degradation when extracting visual features for concepts. This is because the vision encoder and the LLM in an LVLM are jointly trained during pretraining, which partially aligns their semantic representations. As a result, embeddings produced by the native vision encoder are better aligned with the LVLM in the latent space, making them more suitable for constructing in-context prompts. At the same time, the two proposed constraints are more effective when applied to prompts that are better aligned with the LLM decoder's latent space. This finding suggests that personalization methods should prioritize designs that leverage the model’s native components. In addition, our multi-layer feature extraction strategy proves effective. By capturing visual representations at different levels of granularity, it allows the ACP to learn how to adaptively utilize visual embeddings. This design compensates for the absence of explicit masking to highlight key regions. When selecting feature extraction layers, it is necessary to balance the preservation of visual details in the embeddings with the degree of feature fusion. Therefore, choosing layers with indices between 2 and 5 from both the beginning and the end of the encoder is a suitable strategy.

\paragraph{Loss coefficients.} We further conduct an ablation study on the loss coefficients $\alpha$ and $\beta$ using LLaVA-NeXT-7B and Qwen3VL-8B. The results are reported in Table~\ref{app:ab2}. Our observations are as follows.
(1) Personalization performance across different tasks remains relatively stable as these two coefficients vary, indicating that ICPT is robust to the choice of these model-specific hyperparameters. At the same time, LVLMs with weaker general capabilities tend to be more sensitive to these coefficients.
(2) The coefficient $\alpha$ has a larger impact on personalization performance than $\beta$, especially in multi-task settings. This highlights the important role of the two geometric constraints and their derived optimization objective in ICPT.
(3) Moderate coefficient values tend to produce the most balanced performance across tasks. For example, on LLaVA-NeXT-7B, reducing $\alpha$ from 0.5 to 0.1 consistently lowers the weighted scores across tasks, while moderate adjustments (e.g., $\beta = 0.7$) can slightly improve certain tasks such as MVQA. A similar pattern appears for Qwen3VL-8B, where moderate $\beta$ values maintain strong performance while avoiding noticeable degradation. This suggests that the loss terms contribute complementary regularization effects, and excessively weakening either constraint slightly harms the balance among tasks.

\begin{table*}[t]
\centering
\caption{
Ablation study on the coefficients of $\mathcal{L}_{Geo}$ and $\mathcal{L}_{Spa}$ in the overall loss.
}
\vspace{-5pt}
\small
\resizebox{\textwidth}{!}{
\begin{tabular}{c|c||ccc|ccc|ccc|ccc}
\toprule
\multirow{2}{*}{\textbf{LVLM}} & 
\multirow{2}{*}{\textbf{Method}}
& \multicolumn{3}{c|}{\textbf{Recognition}}
& \multicolumn{3}{c|}{\textbf{MVQA}}
& \multicolumn{3}{c|}{\textbf{OVQA}}
& \multicolumn{3}{c}{\textbf{Captioning}} \\
\cmidrule(lr){3-5}
\cmidrule(lr){6-8}
\cmidrule(lr){9-11}
\cmidrule(lr){12-14}
& 
& Single & Multi & Weight
& Single & Multi & Weight
& Single & Multi & Weight
& Single & Multi & Weight \\
\midrule

\multirow{6}{*}{LLaVA-NeXT-7B}
& Original  & 0.873 & 0.859 & 0.868 & 0.850 & 0.774 & 0.791 & 0.716 & 0.697 & 0.702 & 0.713 & 0.624 & 0.654 \\
\cmidrule{2-14}
& $\alpha=0.1$ & 0.865 & 0.827 & 0.852  & 0.840 & 0.737 & 0.760 & 0.702 & 0.627  & 0.645 & 0.687&  0.565 & 0.606\\
& $\alpha=0.5$ & 0.875 & 0.847 &0.866 & 0.835 & 0.752 & 0.770& 0.709 & 0.703 & 0.704 & 0.692 & 0.612 & 0.639 \\
\cmidrule{2-14}
& $\beta=0.3$ & 0.857 & 0.847 & 0.854 & 0.835 &0.746 & 0.766 & 0.712& 0.686 &0.692 & 0.719 &  0.614& 0.649\\
& $\beta=0.7$ & 0.864 & 0.852 &0.860 & 0.845& 0.782 & 0.796& 0.694 & 0.683 &0.686 & 0.698 & 0.616 & 0.643\\
\midrule

\multirow{6}{*}{Qwen3VL-8B}
& Original & 0.963 & 0.924 & 0.950 & 0.920 & 0.867 & 0.879 & 0.776 & 0.724 & 0.736 & 0.835 & 0.821 & 0.826\\
\cmidrule{2-14}
& $\alpha=0.4$  & 0.958& 0.898 & 0.938& 0.905 & 0.836 & 0.851 &0.759 & 0.679& 0.698 & 0.823& 0.799& 0.807\\
& $\alpha=0.8$  & 0.947 & 0.886 & 0.927& 0.900 & 0.853 &0.863 & 0.764& 0.687& 0.705 & 0.816& 0.782& 0.793\\
\cmidrule{2-14}
& $\beta=0.2$ & 0.958& 0.908 & 0.941&0.915&0.858 & 0.871& 0.783 & 0.710 & 0.727 &0.843& 0.836 & 0.838\\
& $\beta=0.6$ & 0.967 & 0.913 &0.949 & 0.925 & 0.849 & 0.866& 0.769&0.694 & 0.712 & 0.827 & 0.813& 0.818\\
\bottomrule
\end{tabular}}
\label{app:ab2}
\end{table*}
\paragraph{Dynamic Token Router (DTR) decision threshold.} We investigate the effect of the DTR decision threshold $\tau$ on the performance across the four personalization tasks. As illustrated in Fig. \ref{app:fig4}, varying $\tau$ from 0.3 to 0.7 yields an inverted U-shaped trend, with performance peaking at $\tau=0.5$ for all tasks. When the threshold is lowered to 0.3 or 0.4, the performance exhibits a slight decline. This indicates that retaining an excessive number of visual tokens allows redundant background features to bypass the router, which introduces noise during the LVLM's inference with in-context prompts and increases cross-concept interference. Conversely, increasing the threshold to 0.6 or 0.7 results in a performance drop across all tasks. This suggests that a strict threshold prunes the fine-grained visual evidence necessary for accurate reasoning. Therefore, setting $\tau=0.5$ provides an appropriate information bottleneck, filtering out visual distractions while preserving essential semantic features.

\paragraph{Contextual Variation Memory (CVM) capacity.} We further examine the effect of the CVM capacity $K$ under varying training batch sizes (BS) on the single- and multi-concept OVQA tasks. As shown in Fig. \ref{app:fig5}, the optimal capacity $K$ is coupled with the batch size. For smaller batch sizes (BS=4 and BS=8), performance peaks early at moderate capacities ($K=64$ or $K=128$) and noticeably declines as $K$ increases to 256. In contrast, for a larger batch size (BS=16), performance steadily improves as $K$ increases up to 256. This phenomenon highlights the underlying trade-off between variation subspace diversity and representation staleness. Because the CVM relies on dynamically updated network representations, filling a large memory queue under a small batch size requires a significantly higher number of optimization steps. This forces the model to regularize against outdated features computed from past iterations, providing misaligned geometric constraints that hinder convergence. A larger batch size mitigates this staleness by updating the queue more rapidly, enabling the model to safely benefit from a richer variation subspace. Consequently, the CVM capacity should be scaled in proportion to the batch size to maintain temporal consistency.
\begin{figure}[t]
\centering
\begin{minipage}{0.48\textwidth}
\centering
\includegraphics[width=\linewidth]{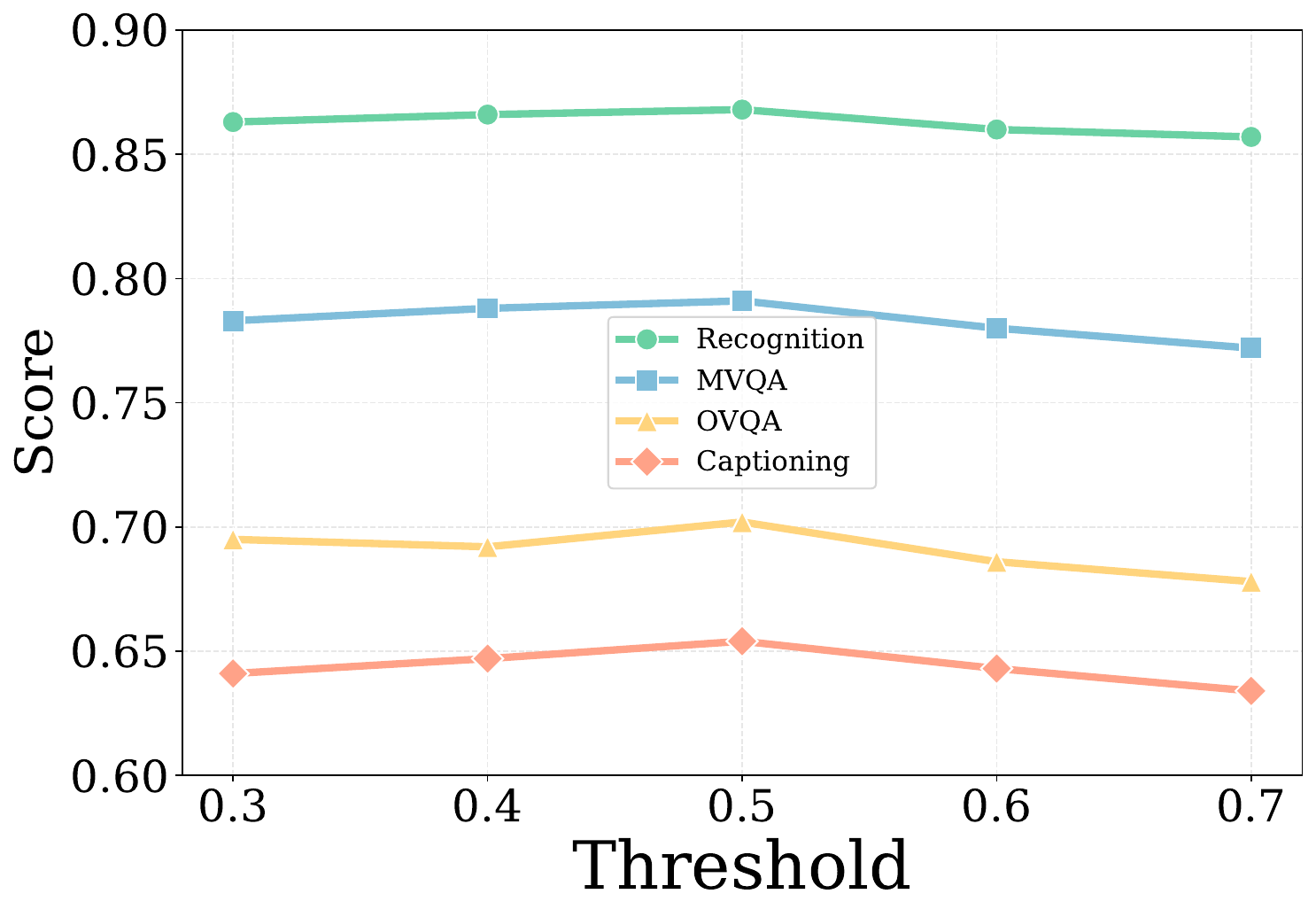}
\captionof{figure}{Ablation study on DTR decision threshold $\tau$.}
\label{app:fig4}
\end{minipage}
\hfill
\begin{minipage}{0.48\textwidth}
\centering
\includegraphics[width=\linewidth]{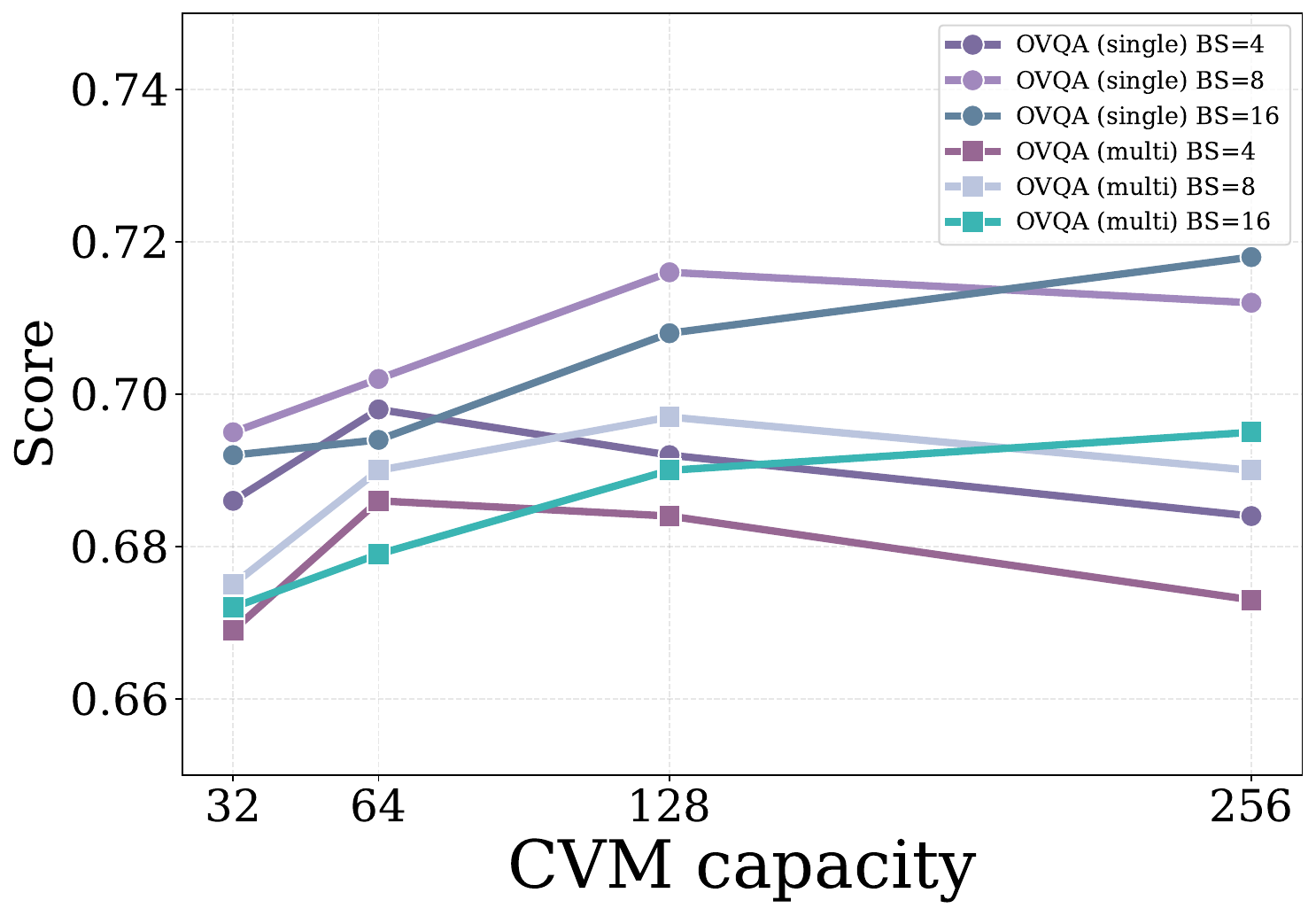}
\captionof{figure}{Ablation study on CVM capacity $K$.}
\label{app:fig5}
\end{minipage}
\end{figure}
\section{Analyses}
\subsection{Additional qualitative examples and failure case study}
We provide two additional qualitative examples of ICPT in Fig. \ref{app:fig1} as a supplement to the examples shown in the main text. The upper example corresponds to data constructed from synthetic images in our test set. The lower example follows the same setting as the examples in the main text, where all images are drawn from real-world visual data. These tasks require strict reliance on visual reasoning over the query image together with an understanding of the personalized concepts. Each case involves multiple concepts, and each concept is associated with a varying number of reference images.
The results show that, with ICPT, the model is able to robustly integrate user-specific concepts into general multimodal reasoning. This capability enables the LVLM to maintain strong reasoning performance while accurately grounding personalized concepts in the visual context, thereby supporting more reliable and flexible personalized multimodal applications.

However, we also observe several representative failure cases of ICPT in the test data, which mainly fall into two categories.
First, in multi-concept settings, some concepts exhibit high visual similarity. In such cases, ICPT may confuse these similar concepts and incorrectly associate the query with the wrong concept label. This behavior suggests that although the geometric constraints help separate concepts in the representation space, visually similar identities can still lead to ambiguity when multiple personalized concepts coexist.
Second, failure cases also arise when the user query contains too many concept labels, or when the generated output needs to include many concept labels. In ICPT, the prompt template explicitly links each textual label with its corresponding word embedding and learned prompts directly. When the number of labels becomes large, this explicit association may make it difficult for the model to correctly interpret or generate the appropriate label, even though the model may still capture the underlying visual features of the corresponding concept.
\begin{figure}[t]
  \centering
  \includegraphics[width=1\linewidth]{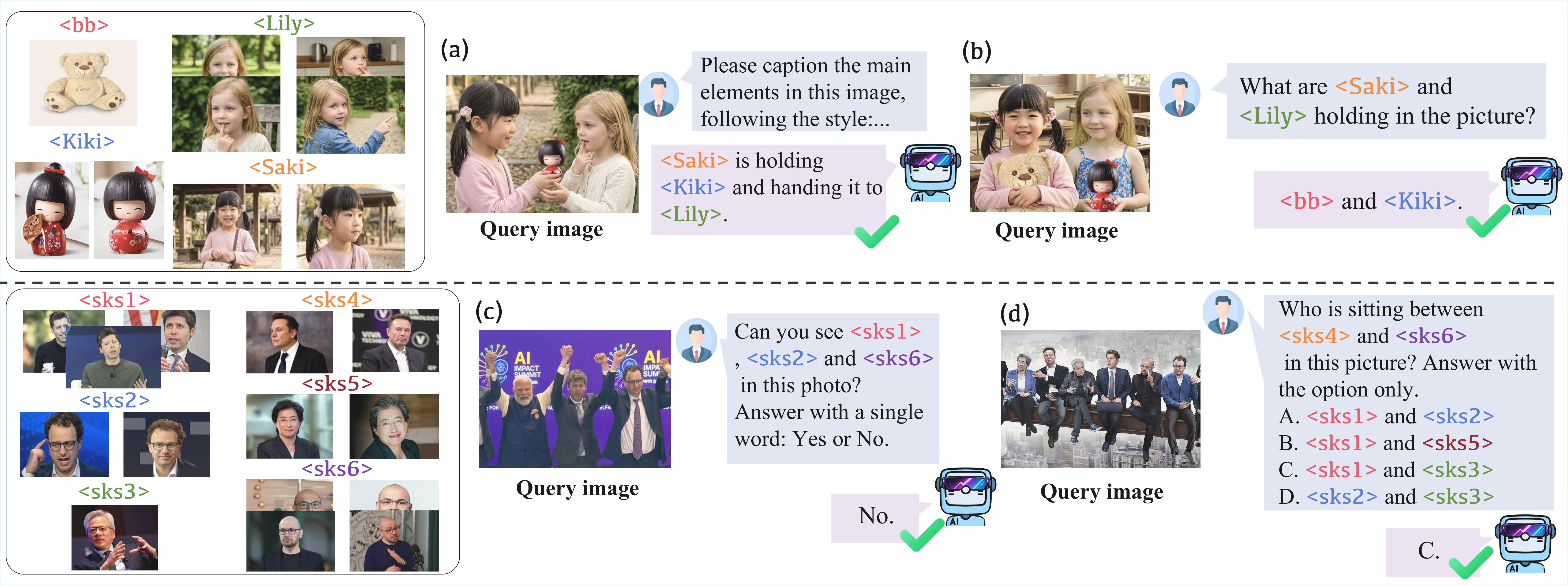}
  \caption{Two additional qualitative examples of ICPT on four personalization tasks.}
  \label{app:fig1}
\end{figure}
\subsection{Efficiency Analyses}
We analyze the efficiency of ICPT from the perspectives of training cost and inference latency. During the training phase, ICPT's lightweight architectural design requires updating only a small number of parameters within the ACP, while the LVLM backbone remains entirely frozen. What's more, the CVM dynamically reuses representations already computed during the ACP's forward pass to populate its queue, completely avoiding the need for redundant network passes. Crucially, ICPT eliminates the need for inference-time training. When encountering novel concepts, vocabulary-expansion methods like MC-LLaVA require separate, computationally expensive fine-tuning steps for each addition. In contrast, ICPT seamlessly extracts and projects new concepts on the fly via a single forward pass. Consequently, as the personalized system is continuously deployed and tasked with acquiring an ever-growing number of concepts, the cumulative savings in training costs become increasingly substantial.

\begin{figure}[t]
\centering
\begin{minipage}{0.48\textwidth}
\centering
\includegraphics[width=\linewidth]{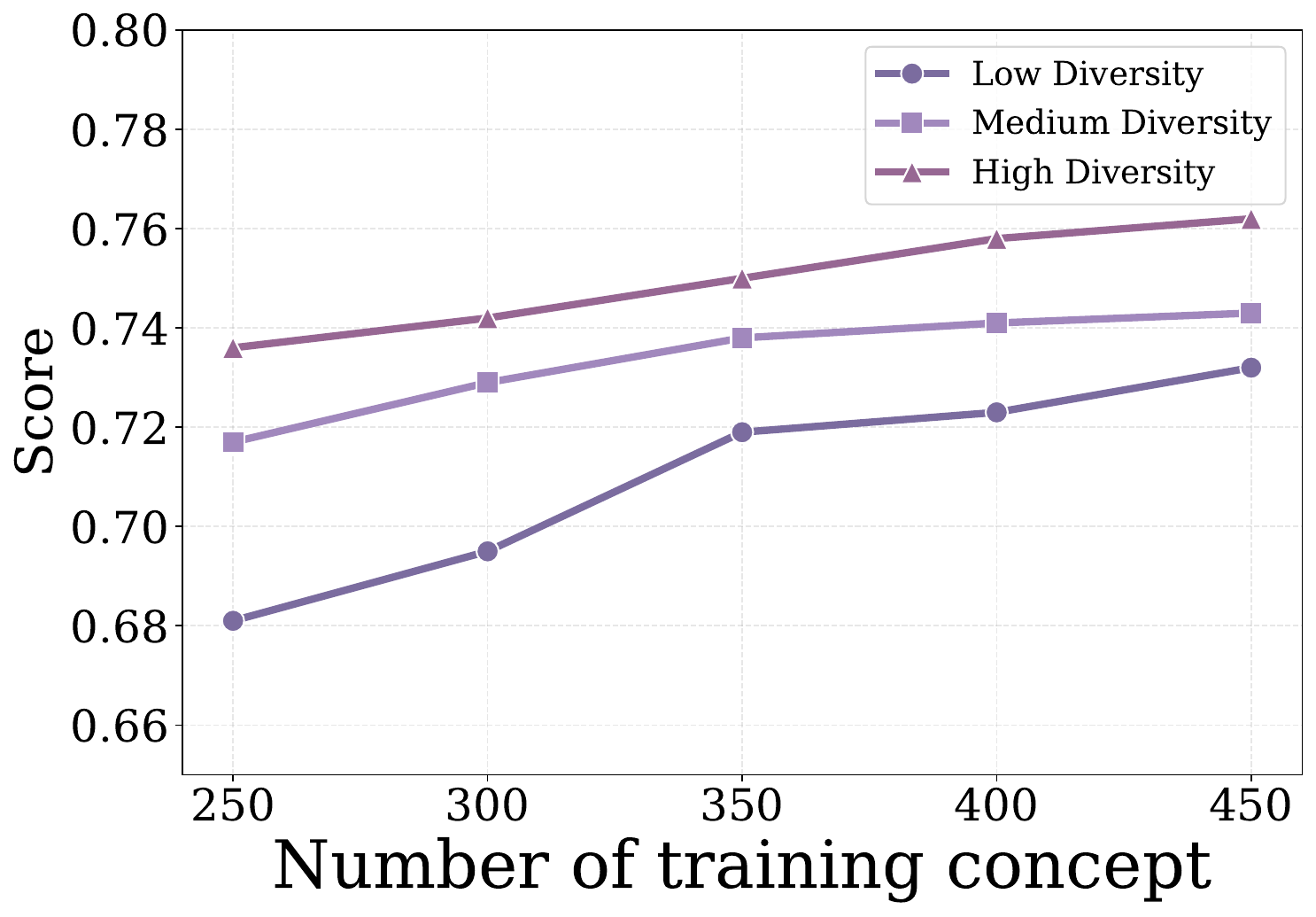}
\captionof{figure}{ICPT's weighted average performance across the four task types under different training data recipes.}
\label{app:fig6}
\end{minipage}
\hfill
\begin{minipage}{0.48\textwidth}
\centering
\captionof{table}{Comparison of the total inference latency of different methods on all test cases and on cases containing four or more concepts.}
\label{app:tab1}
\vspace{3pt}
\begin{tabular}{c c c}
\toprule
\multirow{2}{*}{\textbf{Method}} & \multicolumn{2}{c}{\textbf{Latency}(s)}\\
 & All & $\ge$ 4-concept \\
\midrule
ICL  & 12285.48 & 4631.26\\
PLVM & 2689.95 & 1150.76\\
MC-LLaVA & 2886.78 & 1209.77\\
ICPT & 2367.16 & 943.62\\
\bottomrule
\end{tabular}
\end{minipage}
\end{figure}

Regarding inference efficiency, ICPT is explicitly designed to minimize system latency. This advantage stems directly from two core architectural choices. First, rather than relying on computationally heavy external vision encoders or auxiliary mask generators, ICPT directly leverages hierarchical multi-scale features from the LVLM's built-in vision encoder, thereby streamlining the visual extraction pipeline. Second, the overall decoding latency of an LVLM is primarily bottlenecked by the sequence length of the input prompt. The DTR effectively mitigates this bottleneck by adaptively evaluating the intrinsic visual complexity of each concept and softly pruning redundant visual tokens. This reduces KV-cache memory consumption and attention computation, ensuring rapid response times even in highly complex multi-image and multi-concept scenarios. As reported in Table \ref{app:tab1}, these architectural advantages translate into substantial empirical latency reductions, and the efficiency gap widens as scenarios become more complex with longer input token sequences.

\subsection{Volume vs Diversity in Personalization Training}
In the field of LVLM personalization, determining the optimal training data recipe remains an open but important problem. This is primarily because collecting high-quality, user-specific reference images across varying natural conditions and manually annotating complex multi-concept reasoning interactions are highly labor-intensive and resource-demanding processes. Consequently, it is crucial to understand whether personalization models benefit more from simply scaling up the absolute number of unique training concepts (data volume) or from exposure to highly varied contextual conditions (data diversity).

To systematically evaluate this, we design a controlled scaling experiment. We scale the training data volume, represented by the number of unique training concepts, from 250 to 450. The numbers of query images and QA pairs constructed thereafter are increased proportionally. Concurrently, we categorize data diversity into three distinct regimes-Low, Medium, and High-based on three explicit dimensions: the visual variance among reference images of the same concept, the distribution of the number of concepts per input, and the variety of reasoning tasks. The Low Diversity regime restricts the training data to visually homogeneous reference images, strictly single-concept inputs, and uniform task types. The Medium Diversity regime introduces moderate visual variance and a partial mix of single- and multi-concept inputs. The High Diversity setting corresponds exactly to our current training data construction strategy.

As illustrated in Fig. \ref{app:fig6}, the performance trends across the three regimes reveal a clear hierarchy: High > Medium > Low Diversity across all volume scales. While increasing the data volume generally yields performance improvements, the Low and Medium Diversity curves exhibit clear signs of diminishing returns, plateauing noticeably after reaching 350 concepts. This indicates that passively memorizing an expanding vocabulary of isolated, homogeneous concepts fails to generalize to complex, unseen scenarios. In stark contrast, the High Diversity curve maintains a consistent upward trajectory. Most notably, the high-diversity model trained on merely 250 concepts outperforms the low-diversity model trained on the maximum volume of 450 concepts. This empirically demonstrates that exposing the model to complex visual contexts, varied concept combinations, and diverse task formats is fundamentally more critical than simply expanding the training concept vocabulary. As the foundational capabilities of LVLMs continue to advance, an ideal personalization method should transcend the memorization of massive data. It must act as a robust semantic bridge that distills disentangled concepts from complex real-world inputs and seamlessly aligns them with the model's inherent reasoning capabilities.
%
%

\end{document}